\documentclass[11pt]{article}
\usepackage[margin=2.55cm]{geometry}
\usepackage{amsmath,amssymb,amsthm,mathtools}
\usepackage{graphicx,booktabs,enumitem}
\usepackage{placeins}
\usepackage{microtype}\microtypesetup{expansion=false}
\usepackage{xcolor}
\usepackage[colorlinks=true,linkcolor=blue!45!black,citecolor=blue!45!black,urlcolor=blue!45!black]{hyperref}
\usepackage[numbers,sort&compress]{natbib}

\newtheorem{theorem}{Theorem}[section]

\newtheorem{corollary}[theorem]{Corollary}
\newtheorem{lemma}[theorem]{Lemma}
\theoremstyle{definition}
\newtheorem{definition}[theorem]{Definition}

\theoremstyle{remark}
\newtheorem{remark}[theorem]{Remark}

\newcommand{\R}{\mathbb{R}}
\newcommand{\ghat}{\widehat g}

\title{\bfseries Fractional Laplace Neural Operators:\\
Exact Architectures, an Expressivity Frontier at Criticality,\\
and Certified Stability for Memory-Driven Network Dynamics}
\author{Mauricio Herrera-Mar\'in\\[1mm]
\small Facultad de Ingenier\'ia, Universidad del Desarrollo, Santiago, Chile\\
\small \texttt{mherrera@udd.cl}}
\date{}

\begin{document}
\maketitle

\begin{abstract}
Neural operators learn maps between function spaces, while hereditary network
dynamics are naturally described by Volterra resolvents with non-rational
Laplace symbols. We introduce a fractional Laplace neural operator (fLNO) that
embeds this structure directly in the learned map. For commuting
excitation--Laplacian pairs, one block graph-spectral layer represents the full
linear Volterra solution operator exactly. We establish an expressivity
frontier for finite rational realizations: they approximate fractional memory
geometrically on compact frequency windows, but finite pole sets cannot
reproduce the non-integer critical asymptotics generated by a branch point; on
the half-line the best rational rate is root-exponential. The same theory
yields signed-modal and Perron-normalized trainable parametrizations that
enforce a prescribed stability margin by construction. A graphon-transfer
theorem distinguishes genuine operator consistency from parameter sharing
alone. In a common-data benchmark, positive rational operators can match or
exceed fLNO accuracy on finite horizons, whereas fLNO retains a more faithful
explicit branching coordinate near criticality with substantially fewer
parameters. A scarce/noisy stress test shows that unconstrained rational fits
can cross the stability boundary, while certified parametrizations cannot. A
fixed four-parameter spectral law transfers without retraining from graphs of
size 48 to 192 with $0.51$--$0.62\%$ relative error. Applications to four
quality-controlled Chilean aftershock sequences and renewal models for Chile
and 21 Italian regions illustrate structured inference with explicit
uncertainty and out-of-sample checks. The contribution is not universal
predictive superiority, but an operator-learning architecture in which exact
memory structure, parsimonious physical coordinates, and spectral stability
guarantees coexist with competitive finite-horizon accuracy.
\end{abstract}


\section{Introduction}
\label{sec:intro}

Neural operators learn maps between functions rather than maps between
finite-dimensional vectors. Fourier neural operators (FNOs)
\citep{li2021fno,kovachki2023}, DeepONet \citep{lu2021deeponet}, and more
recent geometry-aware and physics-informed variants have established this
viewpoint as a general framework for discretization-aware scientific
learning. The Laplace neural operator (LNO) \citep{cao2024lno} is
especially relevant to transient systems because it parameterizes the
temporal stage through a pole--residue representation in the Laplace
domain. Finite-dimensional linear state-space sequence models realize the
same rational transfer-function class. Neural operators are also beginning
to be used for system-level tasks such as fixed-point, stability, and
bifurcation analysis rather than only as time-stepping surrogates
\citep{fabiani2026}. Our objective is therefore not to claim that stability
analysis is absent from operator learning. The narrower question addressed
here is whether the \emph{exact long-memory singularity and the exact
Volterra branching threshold} can be encoded by construction in a learned
operator.

The mathematical starting point is the published Volterra theory of
memory-driven network cascades \citep{herrera2026cnsns}. It studies
\begin{equation}
\label{eq:volterra}
U'(t)=\mathcal{B}U(t)+\int_0^t\mathcal{G}_z(t-s)U(s)\,ds+F_z(t),
\qquad
\mathcal{G}_z(t)=
\begin{pmatrix}0&0\\ A_z g_z(t)&A_z g_z(t)\end{pmatrix},
\end{equation}
on \(\mathcal X=\R^n\times\R^n\), with
\(\mathcal B=\operatorname{diag}(-c_xI-\kappa L_G,-c_\lambda I)\), graph
Laplacian \(L_G\), excitation matrix \(A_z\), and a completely monotone
memory kernel \(g_z\) of mass \(G_z\). That analysis establishes the
sharp commuting threshold \(a_{z,j}G_z<c_\lambda\), the
Perron--Frobenius threshold \(r(A_z)G_z<c_\lambda\) for nonnegative
noncommuting excitation, the critical asymptotics of finite-mass kernels,
consistent BDF2 convolution quadrature, and a regular-graphon limit. The
present work uses those results as prior mathematics and asks how much of
them can be preserved inside an operator-learning architecture rather than
re-derived as post-hoc diagnostics.

The key observation is analytic. A tempered fractional memory kernel has
Laplace symbol
\[
\widehat g_{\alpha,\theta}(\xi)=(\xi+\theta)^{\alpha-1},
\qquad 0<\alpha<1,
\]
which has a branch point rather than finitely many poles. A free rational
LNO or a finite linear state-space model can approximate this symbol very
well on a bounded frequency interval, but it cannot represent its branch
singularity exactly. That distinction is negligible for some interpolation
tasks and decisive at criticality, where the non-analytic low-frequency
structure determines the long-time universality class. We therefore build
the non-rational component into the layer as a dedicated fractional atom
and reserve rational terms for finite-memory corrections.

A second observation concerns certification. In this model class the
stability boundary is already known exactly. Rather than learning an
unconstrained operator and testing stability afterward, the trainable gains
can be parameterized so that every point visited by optimization remains
inside a prescribed Volterra stability margin. This is stronger than a
frequency-grid diagnostic and more specific than a generic spectral
regularizer: it enforces the analytic threshold of the underlying
hereditary system.

\subsection{Contributions}
\label{sec:contrib}

\begin{enumerate}[label=\textbf{C\arabic*.},leftmargin=1.05cm,itemsep=2pt]
\item \textbf{Exact block fLNO representation}
(Theorem~\ref{thm:exact}). Under simultaneous orthogonal diagonalization,
the complete \((x,\lambda)\) solution operator of \eqref{eq:volterra} is a
one-layer block-triangular fLNO. The modal transfer matrix contains the
exact exponential state resolvent and the exact fractional intensity
resolvent; no surrogate approximation is required.
\item \textbf{Correct rational expressivity frontier}
(Theorem~\ref{thm:qualitative} and Corollary~\ref{cor:stieltjes-critical}).
A general finite rational realization can generate only integer-power
exponential terms and therefore cannot reproduce exactly fractional
critical asymptotics. For the positive Stieltjes rational class relevant
to stable Markov approximations, the critical zero is simple and the
response approaches a positive neutral level, making the contrast with
untempered fractional critical relaxation explicit.
\item \textbf{Quantitative approximation rates}
(Theorem~\ref{thm:rates}). Positive rational approximants converge
geometrically on compact frequency windows, whereas the half-line problem
reduces to Stahl's approximation of a fractional power and is
root-exponential in the pole count.
\item \textbf{Stability by construction}
(Theorem~\ref{thm:cert}). Signed modal gains are parameterized over the
entire strict subcritical half-line, including arbitrarily strong
inhibition. Nonnegative matrix excitation is Perron-normalized with an
independent trainable radial factor so that the full interior of the
prescribed stability region is accessible. Positive rational corrections
are certified through their exact zero-frequency mass, not by a frequency
grid.
\item \textbf{Graphon transfer for a consistent learned family}
(Theorem~\ref{thm:transfer}). Transfer bounds assume both a genuine
operator approximation error on a bounded input class and consistency of
the size-independent learned family with a common graphon operator. Thus a
small trajectory loss is never identified with an operator norm, and a
time-uniform bound is obtained only under the stronger contraction
certificate.
\item \textbf{Direct common-data operator benchmark}. On identical
forcing--response maps and train/validation/test splits, we compare fLNO
against positive and unconstrained rational pole--residue operators, a
diagonal linear state-space model, FNO-1D, and DeepONet. The benchmark
separates finite-window prediction accuracy, doubled-horizon extrapolation,
branching-margin recovery, parameter count, and graph-size transfer.
\item \textbf{Quality-controlled applications}. The seismic illustration
uses four Chilean sequences extracted by one reproducible pipeline from the
relocated CSN catalog of \citet{potin2024}, with completeness diagnostics,
profile likelihood, a 300-replicate bootstrap sensitivity analysis, and
ordered along-strike spatial operators. The catalog-processing principles
are consistent with the broader homogenization framework of
\citet{candia2026}. A separate epidemic application treats the scalar and
network renewal equations as instances of the same branching architecture,
while distinguishing mathematical certification from statistical
uncertainty.
\end{enumerate}

\subsection{Related work}
\label{sec:related}

\emph{Neural operators.} FNO \citep{li2021fno}, DeepONet
\citep{lu2021deeponet}, and the general neural-operator framework
\citep{kovachki2023} motivate the function-to-function viewpoint. LNO
\citep{cao2024lno} is the closest architectural precursor because its
Laplace pole--residue stage is a rational transfer function designed for
transients and extrapolation. Recent work has made memory and state-space
structure increasingly explicit in operator learning: MemNO combines S4 and
FNO to exploit history in time-dependent PDEs \citep{buitrago2025memno},
while the Mamba Neural Operator develops a direct connection between neural
operators and structured state-space models \citep{cheng2026mamba}. These
approaches motivate finite-dimensional learned memory; our expressivity result
instead isolates the non-integer branch structure that finite rational
realizations can approximate on compact windows but cannot reproduce exactly
at criticality. PINO \citep{li2024pino}, PINTO \citep{pinto2025}, and other
physics-informed operators incorporate governing structure into learning. In
particular, the recent Physics-Informed Laplace Neural Operator augments the
LNO pole--residue backbone with PDE, boundary, and initial-condition residuals
\citep{kim2026pilno}; this is complementary to the present architectural
use of an exact fractional atom and an analytic branching-stability
certificate. Geometric operator learning addresses transfer across varying
geometries \citep{li2025geo}. Fabiani et al. \citep{fabiani2026} recently
showed that local neural operators can support fixed-point, stability, and
bifurcation analysis. The distinction here is narrower: the stability
certificate is not extracted from a learned time-stepper; it is enforced
analytically by the exact branching threshold of a completely monotone
Volterra operator.

\emph{Rational and state-space approximations.} Deep linear state-space
models such as S4 \citep{gu2022s4} and pole--residue LNO layers realize
finite rational kernels. Multifactor Markovian approximations are also
standard for fractional models in rough volatility \citep{abijaber2019}.
The present approximation results quantify when this finite-pole class is
sufficient and when the branch point matters asymptotically.

\emph{Approximation theory.} The compact-window construction is a
positive Gauss approximation of a Markov function. The half-line reduction
uses Stahl's rational approximation of fractional powers \citep{stahl2003};
the constructive upper bound follows the exponential-sum philosophy of
Braess--Hackbusch \citep{braess2005}.

\emph{Seismic data.} The empirical section builds on the Volterra--Hawkes
calibration in \citet{herrera2026cnsns} but adopts a stricter data-quality
policy. In particular, recent work by \citet{candia2026} documents a
homogenized Chilean catalog integrating 12 sources, explicit duplicate
resolution, tectonic classification, completeness analysis, and multiple
declustering schemes. Our Hawkes application uses the same conceptual
separation between duplicate removal and dependent-event removal: aftershock
clustering is signal for a self-exciting model and must not be declustered
away.

\subsection{Notation and organization}

\(\ghat_{\alpha,\theta}(\xi)=(\xi+\theta)^{\alpha-1}\) denotes the
Laplace symbol of the tempered fractional kernel
\(g_{\alpha,\theta}(t)=t^{-\alpha}e^{-\theta t}/\Gamma(1-\alpha)\), with
mass \(G(\theta,\alpha)=\theta^{\alpha-1}\). The empirical Omori exponent
is denoted by \(p\), not \(\alpha\), to avoid conflating two distinct
parameterizations. Section~\ref{sec:arch} defines the architecture;
Sections~\ref{sec:exact}--\ref{sec:transfer} give the theory;
Section~\ref{sec:experiments} reports controlled experiments;
Sections~\ref{sec:seismic}--\ref{sec:epi} contain the applications; and
Section~\ref{sec:discussion} states limitations.

\section{The fractional Laplace neural operator}
\label{sec:arch}

\begin{definition}[Fractional and positive rational atoms]
\label{def:atom}
A \emph{fractional atom} has Laplace symbol
\(\ghat_{\alpha,\theta}(\xi)=(\xi+\theta)^{\alpha-1}\),
\(0<\alpha<1\), \(\theta>0\). A \emph{positive rational correction} is
\[
H_j(\xi)=\sum_{k=1}^{N_j}\frac{\beta_{jk}}{\xi+r_{jk}},
\qquad \beta_{jk}\ge0,\quad r_{jk}>0.
\]
Both are Stieltjes functions. Their combined zero-frequency mass is
\[
G^{\rm tot}_j=G(\theta,\alpha)+H_j(0)
=\theta^{\alpha-1}+\sum_k\frac{\beta_{jk}}{r_{jk}}.
\]
The empirical sections also use Omori-type atoms, but their exponent is
denoted by \(p\) to keep it distinct from the fractional singularity
parameter \(\alpha\).
\end{definition}

\begin{definition}[Block fractional Laplace neural operator]
\label{def:flno}
Let \(L_G\Phi=\Phi\operatorname{diag}(\ell_j)\) and suppose, when a modal
representation is used, that \(A\Phi=\Phi\operatorname{diag}(a_j)\).
For each mode set
\[
\beta_{x,j}=c_x+\kappa\ell_j,\qquad
\Psi_j(\xi)=\ghat_{\alpha,\theta}(\xi)+H_j(\xi),\qquad
D_j(\xi)=\xi+c_\lambda-a_j\Psi_j(\xi).
\]
The one-layer \emph{block fLNO} is the modal transfer matrix
\begin{equation}
\label{eq:blockflno}
\widehat{\mathcal R}_j(\xi)=
\begin{pmatrix}
(\xi+\beta_{x,j})^{-1} & 0\\[2mm]
\dfrac{a_j\Psi_j(\xi)}{(\xi+\beta_{x,j})D_j(\xi)} & D_j(\xi)^{-1}
\end{pmatrix}.
\end{equation}
Applied to the modal inputs
\(b_{x,j}=x_{0,j}+\widehat f_{x,j}\) and
\(b_{\lambda,j}=\lambda_{0,j}+\widehat f_{\lambda,j}\), it produces
\((\widehat x_j,\widehat\lambda_j)^\top
=\widehat{\mathcal R}_j(b_{x,j},b_{\lambda,j})^\top\).

For noncommuting nonnegative excitation we use the corresponding
\emph{matrix fLNO} intensity resolvent
\begin{equation}
\label{eq:matrixflno}
\widehat{\mathcal R}_{A}(\xi)
=\big[(\xi+c_\lambda)I-A\Psi(\xi)\big]^{-1},
\end{equation}
with the upstream \(x\)-state entering as a history-dependent forcing.
The diagonal/modal layer is therefore an efficient specialization of the
matrix layer, not a separate model class. For the operator-theoretic recovery
and convergence experiments we use consistent BDF2 convolution quadrature
\citep{lubich1988}. The common-data learning benchmark in
Section~\ref{sec:benchmark} deliberately uses an independently implemented
direct time-domain product-integration recurrence; the two discretizations
are kept distinct rather than conflated.
\end{definition}

\begin{remark}[Why this is a neural operator]
The fLNO is a trainable map $\mathcal N_\vartheta$ between input and output
function spaces: the same learned parameters act on arbitrary forcing
histories rather than on a fixed finite collection of vectors. In the
size-independent graph-filter form $A=f_\vartheta(L_G)$, the same parameter
vector is also shared across graph discretizations. The present paper
intentionally isolates a certified linear hereditary core; it does not claim
universal nonlinear operator approximation by a single layer.
\end{remark}

The architecture separates three objects that play different roles:
(1) the exact non-rational memory atom, (2) optional positive rational
corrections for finite-memory structure, and (3) a graph or matrix
excitation operator whose spectral edge determines stability. In a
transferable graph-filter realization the modal gains may be parameterized
with a size-independent function \(a_j=f_\vartheta(\ell_j)\), rather than
as one free parameter per mode. This distinction is important: free
\(\{a_j\}\) gives exact finite-graph representation, whereas
\(f_\vartheta\) gives a fixed-dimensional family suitable for transfer
across discretizations.

\section{Exactness on the commuting class}
\label{sec:exact}

\begin{theorem}[Exact representation of the full product-space operator]
\label{thm:exact}
Assume the commuting hypothesis of \citet[Assumption 4.1]{herrera2026cnsns}
and let \(g=g_{\alpha,\theta}\). With \(H_j\equiv0\), the complete
solution operator of \eqref{eq:volterra} is exactly the block fLNO
\eqref{eq:blockflno} with the physical parameters
\(a_j,c_x,\kappa,c_\lambda,\alpha,\theta\). Hence the finite-graph
operator is represented with zero approximation error. If the modal gains
are free, the representation uses \(n+5\) scalar physical parameters;
if \(A=f_\vartheta(L_G)\), the parameter count is
\(\dim(\vartheta)+5\), independent of \(n\).
\end{theorem}

\begin{proof}
In the common orthonormal basis, the state equation gives
\[
\widehat x_j(\xi)=\frac{x_{0,j}+\widehat f_{x,j}(\xi)}
{\xi+c_x+\kappa\ell_j}.
\]
The intensity equation gives
\[
\widehat\lambda_j(\xi)=
\frac{\lambda_{0,j}+\widehat f_{\lambda,j}(\xi)
+a_j\ghat_{\alpha,\theta}(\xi)\widehat x_j(\xi)}
{\xi+c_\lambda-a_j\ghat_{\alpha,\theta}(\xi)}.
\]
Substitution of the first identity into the second gives exactly the
matrix \eqref{eq:blockflno}. No approximation step is used.
\end{proof}

\begin{remark}[Identifiability-friendly invariant]
\label{rem:ident}
The physically decisive trained quantity is
\(\rho_j=a_j\theta^{\alpha-1}\). In the controlled recovery experiment
this branching quantity is substantially better conditioned than
\(\alpha\) and \(\theta\) separately, which trade off along a shallow
likelihood direction.
\end{remark}

\section{The expressivity frontier of rational classes}
\label{sec:gap}

Finite-dimensional linear state-space realizations and pole--residue LNO
blocks have rational transfer functions. Rational approximation can be
excellent on finite windows, but exact fractional critical asymptotics
require a branch point. The following statement deliberately separates a
general rational obstruction from the stronger conclusion available for
positive Stieltjes kernels.

\begin{theorem}[Fractional critical asymptotics are not exactly representable by finite rational realizations]
\label{thm:qualitative}
Let \(H\) be any proper rational function with all poles in
\(\{\Re\xi<0\}\), and suppose \(H(0)=G\). At the critical gain
\(aG=c_\lambda\), the transfer function
\[
\widehat\lambda(\xi)=\frac{\lambda_0}
{\xi+c_\lambda-aH(\xi)}
\]
is rational. Consequently its inverse Laplace transform is a finite sum
of terms \(t^k e^{\mu t}\), with integer \(k\ge0\) and finitely many
\(\mu\). Therefore no finite rational realization can reproduce exactly
(a) an untempered fractional critical law
\(Ct^{\gamma-1}\), \(0<\gamma<1\), or (b) a tempered branch correction
\(Ce^{-\theta t}t^{\alpha-2}\), \(0<\alpha<1\). A multiple zero at the
origin may generate an integer-power polynomial term, so no universal
neutral-limit statement is made for an arbitrary signed rational
\(H\).
\end{theorem}

\begin{proof}
The denominator is rational, hence \(\widehat\lambda\) is rational after
cancellation of common factors. Partial fractions therefore yield only
finitely many exponential-polynomial terms. Non-integer powers in time
correspond to non-meromorphic singular structure (a branch point) and
cannot be generated by a finite rational function.
\end{proof}

\begin{corollary}[Positive Stieltjes rational critical response]
\label{cor:stieltjes-critical}
Assume in addition
\[
H(\xi)=\sum_{k=1}^{N}\frac{\beta_k}{\xi+r_k},
\qquad \beta_k>0,\quad r_k>0.
\]
Then \(m_1^H:=-H'(0)=\sum_k\beta_k/r_k^2\in(0,\infty)\). At
\(aG=c_\lambda\), the zero at \(\xi=0\) is simple because
\[
\frac{d}{d\xi}\big(\xi+c_\lambda-aH(\xi)\big)\Big|_{\xi=0}
=1+a m_1^H>0.
\]
If all remaining characteristic zeros lie in the open left half-plane,
then
\[
\lambda(t)\longrightarrow \frac{\lambda_0}{1+a m_1^H}>0
\]
with exponentially damped integer-power transients. Thus the positive
finite-pole Markov class cannot reproduce the algebraic decay to zero of
an untempered regularly varying critical mode.
\end{corollary}

\begin{theorem}[Quantitative rates: geometric on windows, Stahl on the
half-line]
\label{thm:rates}
Let $\mathcal S_N^{+}=\{\sum_{k\le N}\beta_k/(\xi+r_k):\beta_k\ge0,
r_k>0\}$ and let $E_N^{+}(E)$ be its best sup-error for
$\ghat_{\alpha,\theta}$ on $E$.
\begin{enumerate}[label=\textup{(\roman*)},leftmargin=0.85cm]
\item \textup{(Compact windows.)} For every $\Xi>0$, with
$\tau=\theta/\Xi$ and $\rho'=1+\tau+\sqrt{\tau(2+\tau)}$,
\[
E_N^{+}\big([0,\Xi]\big)\le
\frac{8\,\theta^{\alpha-1}}{1-\rho'^{-1}}\,\rho'^{-2N}.
\]
An explicit Gauss--Jacobi approximant has positive residues and poles in
$(\theta,\infty)$. Hence the pole count required for a fixed tolerance is
logarithmic in the tolerance on a fixed window, while it grows as the
window expands.
\item \textup{(Half-line.)} The substitution
$u=\theta/(\xi+\theta)$ identifies the problem with rational
approximation of $u^{1-\alpha}$ on $[0,1]$. Stahl's asymptotics
\citep{stahl2003} imply a root-exponential lower rate for arbitrary
rational degree-$N$ approximants, while a dyadic positive construction in
$\mathcal S_N^{+}$ attains $E_N^{+}([0,\infty))\le C e^{-c_1\sqrt N}$,
with $c_1=\sqrt{2(\log3)(\log2)(1-\alpha)}$. Thus the half-line problem
has the sharp scaling $\log E_N\asymp-\sqrt N$.
\end{enumerate}
\end{theorem}

Complete proofs are given in Appendix~\ref{app:rates};
Figure~\ref{fig:thm42} displays the compact-window and half-line regimes.
The architectural point is not that rational layers are poor approximants
on finite windows; they are often excellent. The distinction is that a
finite rational model cannot reproduce the branch singularity exactly and
requires a growing number of poles as the effective horizon expands.

\begin{figure}[t]
\centering
\includegraphics[width=\textwidth]{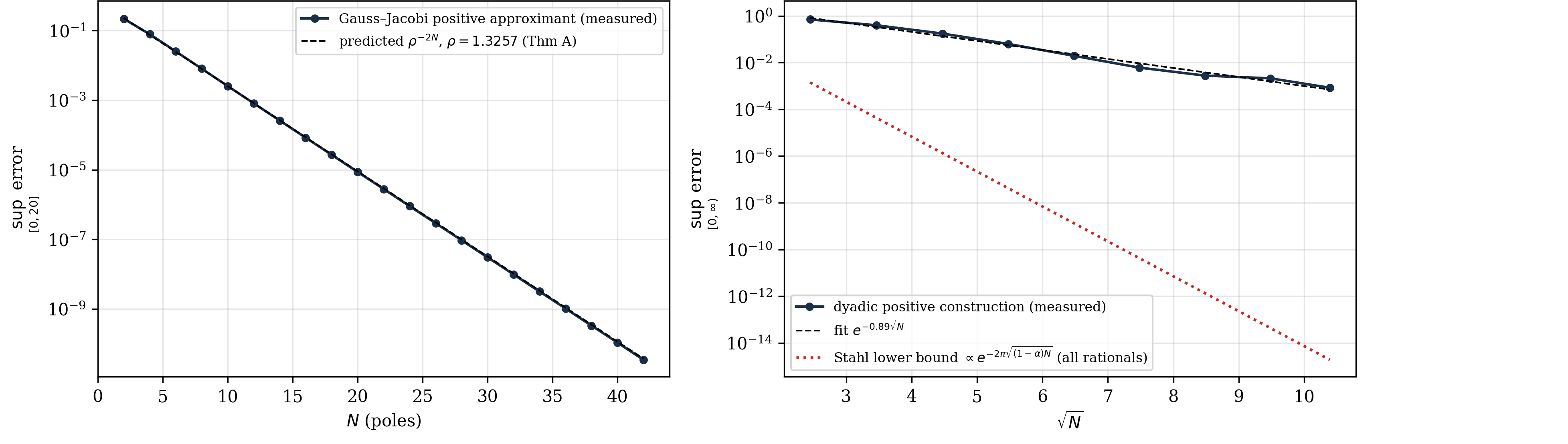}
\caption{Rational approximation of the tempered fractional symbol at
$(\alpha,\theta)=(0.7,0.4)$. \textbf{(A)} On the compact window
$[0,20]$, the positive Gauss--Jacobi approximation converges geometrically.
\textbf{(B)} On the half-line, the dyadic positive construction exhibits
the root-exponential $e^{-c\sqrt N}$ scaling predicted by the rational
approximation theory.}
\label{fig:thm42}
\end{figure}

\section{Stability-certified parametrizations}
\label{sec:cert}

\begin{theorem}[Hard subcriticality by construction]
\label{thm:cert}
Fix a margin \(m\in(0,1)\). For a modal layer define its total positive
memory mass by \(G^{\rm tot}_j=G+H_j(0)\), with the rational correction in
the positive Stieltjes form of Definition~\ref{def:atom}. Then either of
the following parametrizations is subcritical for every trainable weight.
\begin{enumerate}[label=\textup{(\alph*)},leftmargin=0.85cm]
\item \textup{(signed modal)}
\[
a_j=\frac{(1-m)c_\lambda}{G^{\rm tot}_j}
-\operatorname{softplus}(w_j),\qquad w_j\in\R;
\]
this spans the entire strict half-line
\(a_j<(1-m)c_\lambda/G^{\rm tot}_j\), including arbitrarily strong
inhibitory gains;
\item \textup{(nonnegative matrix)} choose a fixed nonnegative support
mask \(M\), let
\(\widetilde A=M\odot\operatorname{softplus}(W)\), and set
\[
A=q\,\frac{(1-m)c_\lambda}
{G^{\rm tot}\,r(\widetilde A)}\,\widetilde A,
\qquad q=\sigma(v)\in(0,1).
\]
Then \(r(A)G^{\rm tot}< (1-m)c_\lambda\). For any fixed support pattern,
the closure of this parametrization contains every nonnegative matrix
satisfying that strict margin.
\end{enumerate}
In both cases the continuous-time Volterra system has margin at least
\(m c_\lambda\). Under the stated Stieltjes/subcritical assumptions, the consistent BDF2
convolution-quadrature symbol remains in the closed right half-plane, so the
continuous zero-free characteristic argument is preserved at the discrete
level. Positive-realness is guaranteed
analytically by the atom parametrization \(\beta_{jk}\ge0,r_{jk}>0\);
a frequency grid may be used as a diagnostic but is not a mathematical
certificate.
\end{theorem}

\begin{remark}[What the certificate does not cover]
If an unconstrained rational correction with signed residues is added,
its contribution to the zero-frequency mass and positive-real property
must be certified separately (for example by an analytic positive-real
realization test). The simple branching certificate based only on the
fractional mass \(G\) is not valid after adding an arbitrary \(H_j\).
\end{remark}

The practical advantage of the parametrization is that stability is a
property of the model class, not a penalty that must be checked after
training. The modal map does not artificially cap inhibition, and the
matrix map has an independent radial coordinate \(q\), so optimization is
not confined to a single Perron-radius surface.

\section{Graphon transferability}
\label{sec:transfer}

Graphon limits provide a natural language for graph-size transfer, but parameter sharing alone is not a convergence theorem \citep{ruiz2020,ruiz2023}.

\begin{theorem}[Finite-horizon graphon transfer for a consistent learned family]
\label{thm:transfer}
Let $W$ be a regular symmetric graphon in the class considered by
\citet[Section~7]{herrera2026cnsns}. For a compact certified parameter set
$K$, let $\mathcal S_k$ be the true size-$k$ solution operator and
$\mathcal N_k^\vartheta$ the corresponding size-independent fLNO family,
acting on a bounded input class $\mathcal B$. Assume that there are common
limit operators $\mathcal S_W$ and $\mathcal N_W^\vartheta$ such that,
uniformly for $\vartheta\in K$,
\[
\|\mathcal S_k-\mathcal S_W\|_{\mathcal B,T}\le C_S\eta_k,
\qquad
\|\mathcal N_k^\vartheta-\mathcal N_W^\vartheta\|_{\mathcal B,T}
\le C_N\eta_k,
\]
where $\|\cdot\|_{\mathcal B,T}$ denotes the operator norm from
$\mathcal B$ to $C([0,T])$ and $\eta_k\to0$. If parameters
$\widehat\vartheta\in K$ learned at size $n$ satisfy the genuine
operator bound
\[
\varepsilon_{\rm op}:=
\|\mathcal N_n^{\widehat\vartheta}-\mathcal S_n\|_{\mathcal B,T},
\]
then deployment of the same parameters at size $m$ obeys
\[
\|\mathcal N_m^{\widehat\vartheta}-\mathcal S_m\|_{\mathcal B,T}
\le
\varepsilon_{\rm op}+(C_N+C_S)(\eta_n+\eta_m).
\]
If $\eta_k$ is specified at the coefficient/operator level rather than at
the solution-operator level, the same conclusion holds with the usual
finite-horizon Volterra--Gronwall factor $C_0e^{C_1T}$ multiplying the
right-hand side. In neither formulation does a small error on finitely many
training trajectories constitute an estimate of $\varepsilon_{\rm op}$.
\end{theorem}

\begin{remark}[Parameter sharing is not graphon consistency]
Using the same finite-dimensional parameter vector at two graph sizes is
necessary for transfer, but it does not imply the second consistency bound in
Theorem~\ref{thm:transfer}. That bound is an explicit hypothesis and must be
justified by the chosen graph-filter realization and graph sequence. The
numerical experiment below therefore demonstrates parameter transfer only;
it is not presented as a verification of graphon consistency.
\end{remark}

\begin{corollary}[Time-uniform transfer under a strong contraction certificate]
\label{cor:uniform}
Assume in addition that the true and learned graphon families satisfy the
same uniform contraction margin
$\|A\|G\le(1-m)\omega$, where
$\omega=\min(c_x,c_\lambda)$, and that the finite-to-graphon consistency
bounds above hold in the corresponding time-uniform operator norm. Then the
feedback kernel
\[
K(t)=\|A\|\,(e^{-\omega\cdot}*g)(t)
\]
has $L^1$-mass $\|K\|_1=\|A\|G/\omega\le1-m$, the error resolvent is
uniformly summable, and
\[
\|\mathcal N_m^{\widehat\vartheta}-\mathcal S_m\|_{\mathcal B,\infty}
\le \frac{C}{m}
\left[\varepsilon_{\rm op}+(\eta_n+\eta_m)\right],
\]
with $C$ independent of graph size.
\end{corollary}

The synthetic transfer experiment below should therefore be read as a
\emph{physics-parameter transfer} experiment: the size-independent memory
and dissipation parameters are learned on one graph and composed with the
known spectral structure of the deployment graph. This is a useful test of
discretization transfer, but it is not by itself an empirical estimate of
the operator norm \(\varepsilon_{\rm op}\).

\section{Controlled experiments and operator benchmark}
\label{sec:experiments}

\paragraph{Recovery and expressivity (Figure~\ref{fig:proto}).}
Fitting one fractional atom by nonlinear least squares through the
CQ-BDF2 forward map to noisy subcritical trajectories
($\sigma=10^{-2}$, $T=40$) recovers, over 12 replications, $\alpha$ to
$6\times10^{-3}$, $\theta$ to $1.1\times10^{-2}$, the gain to
$2.4\times10^{-3}$, and the branching ratio $\rho$ to
$\mathbf{2.9\times10^{-4}}$ mean absolute error: the stability diagnostic
is the best-identified direction (Remark~\ref{rem:ident}). At criticality
with the heavy-tail kernel ($\gamma=\tfrac12$), a six-pole rational
surrogate fitted to sup-error $4\times10^{-3}$ on $[0,50]$ extrapolates
to the spurious constant $0.032$ while the true intensity decays as
$t^{-1/2}$; a single heavy-tail atom fitted on the same window recovers
$\gamma=0.5000$ --- Theorem~\ref{thm:qualitative} in one picture.

\begin{figure}[t]
\centering
\includegraphics[width=\textwidth]{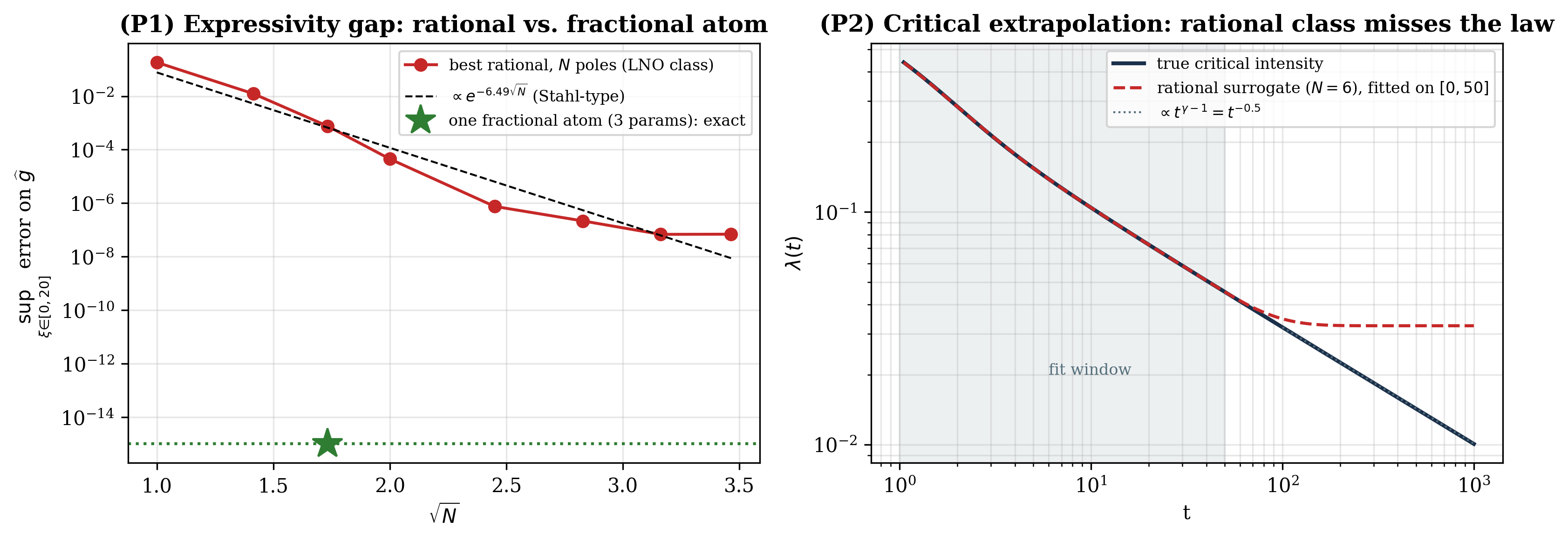}
\caption{\textbf{(P1)} Best $N$-pole Markov rational error on the
tempered symbol over $[0,20]$ versus $\sqrt N$ (a range on which the
compact-window geometric regime is observed), with the exact atom for
reference. \textbf{(P2)} Critical heavy-tail intensity: a rational
surrogate matching the shaded fit window to $4\times10^{-3}$ extrapolates
to a spurious positive constant; the truth follows $t^{\gamma-1}$.}
\label{fig:proto}
\end{figure}

\subsection{Common-data operator benchmark}
\label{sec:benchmark}

We next test whether the structural advantages above survive a direct
operator-learning comparison. Every model receives exactly the same maps
``forcing history $\mapsto$ response history,'' generated from the
fractional Volterra system with
\[
\Delta t=0.1,\qquad T_{\rm train}=9.6,\qquad T_{\rm test}=19.2,
\qquad (\alpha,\theta,c_\lambda,\rho)=(0.65,0.12,1,0.90).
\]
The common split contains 96 training, 24 validation, and 32 test
trajectories. Reference trajectories are generated directly by the causal
time-domain recurrence, using an implicit local-state update and midpoint
quadrature of the fractional memory kernel; no impulse-response convolution
is used to synthesize arbitrary forcings. The forcing is a smooth random
mixture of low-frequency Fourier components, Gaussian pulses, and piecewise
biases and is set to zero after the training horizon, so evaluation on
$(T_{\rm train},2T_{\rm train}]$ is genuine autonomous extrapolation. We report two independent training
seeds. The structured models and FNO-1D are optimized with a 500-epoch cap,
whereas DeepONet uses a 1000-epoch cap; in every case the reported checkpoint
is selected by validation loss. The comparison is therefore common-data and
convergence-oriented rather than an equal-wall-clock benchmark.

The baselines are deliberately named by what was actually implemented. The
``positive rational'' model is a 12-pole positive pole--residue operator,
the ``unconstrained rational'' model has signed free residues, and the
``linear SSM'' is a diagonal 16-state linear realization. They represent
the rational classes underlying LNO/state-space memory layers but are not
claimed to be the authors' official LNO or S4 implementations. FNO-1D and
DeepONet use standard spectral-convolution and branch--trunk constructions.

\paragraph{Computational complexity.}
The structural advantage of fLNO is parameter economy and certification, not
a claim of universal wall-clock superiority. A straightforward history-based
fractional convolution costs $O(N_t^2)$ per scalar trajectory, whereas an
$N_p$-pole Markov/state-space realization admits an $O(N_pN_t)$ recurrence.
Standard fast-convolution and block convolution-quadrature techniques can
accelerate the fractional history evaluation toward near-linear-logarithmic
temporal cost, while graph-spectral transforms add a graph-dependent cost.
The timings in Table~\ref{tab:operator-benchmark} therefore document the
particular implementations used here and should not be read as asymptotic
complexity rankings.

\begin{table}[t]
\centering
\caption{Common-data operator benchmark at $\rho=0.90$. Errors are mean
$\pm$ standard deviation across two seeds. ``Tail'' is the relative
$L^2$ error on $(T_{\rm train},2T_{\rm train}]$. Branch error is the
absolute error of the explicit branching parameter when one exists and the
zero-frequency mass proxy for the unconstrained rational model. No
intrinsic branching coordinate is reported for the generic SSM, FNO, or
DeepONet baselines. Training times are included for scale but are
hardware- and implementation-dependent.}
\label{tab:operator-benchmark}
\small
\begin{tabular}{lrrrrr}
\toprule
Model & Params. & In-window & Tail & Branch error & Train (s)\\
\midrule
fLNO & 3 & $0.140\pm0.025\%$ & $2.20\pm0.16\%$ & $0.0264\pm0.0011$ & 25.2\\
Positive rational & 25 & $0.076\pm0.020\%$ & $2.55\pm0.01\%$ & $0.0374\pm0.0009$ & 19.3\\
Unconstrained rational & 24 & $0.244\pm0.057\%$ & $3.92\pm0.32\%$ & $0.0517\pm0.0080$ & 19.5\\
Linear SSM & 49 & $0.333\pm0.086\%$ & $3.92\pm1.46\%$ & --- & 17.8\\
FNO-1D & 4241 & $3.54\pm0.88\%$ & $489\pm123\%$ & --- & 17.6\\
DeepONet & 18817 & $11.3\pm1.4\%$ & $1170\pm119\%$ & --- & 8.0\\
\bottomrule
\end{tabular}
\end{table}

Table~\ref{tab:operator-benchmark} and Figure~\ref{fig:operator-benchmark}
show a deliberately nuanced result. The positive rational operator is
slightly more accurate than the fLNO inside the training window, confirming
that finite-pole models can approximate fractional memory extremely well on
compact horizons. On the autonomous tail the ordering reverses slightly:
the fLNO has the smallest mean error among the structured models while using
only three trainable parameters. The large FNO-1D and DeepONet errors beyond
$T_{\rm train}$ are specific to this causal-memory benchmark and should not
be read as a general ranking of neural-operator families.

\begin{figure}[t]
\centering
\includegraphics[width=\textwidth]{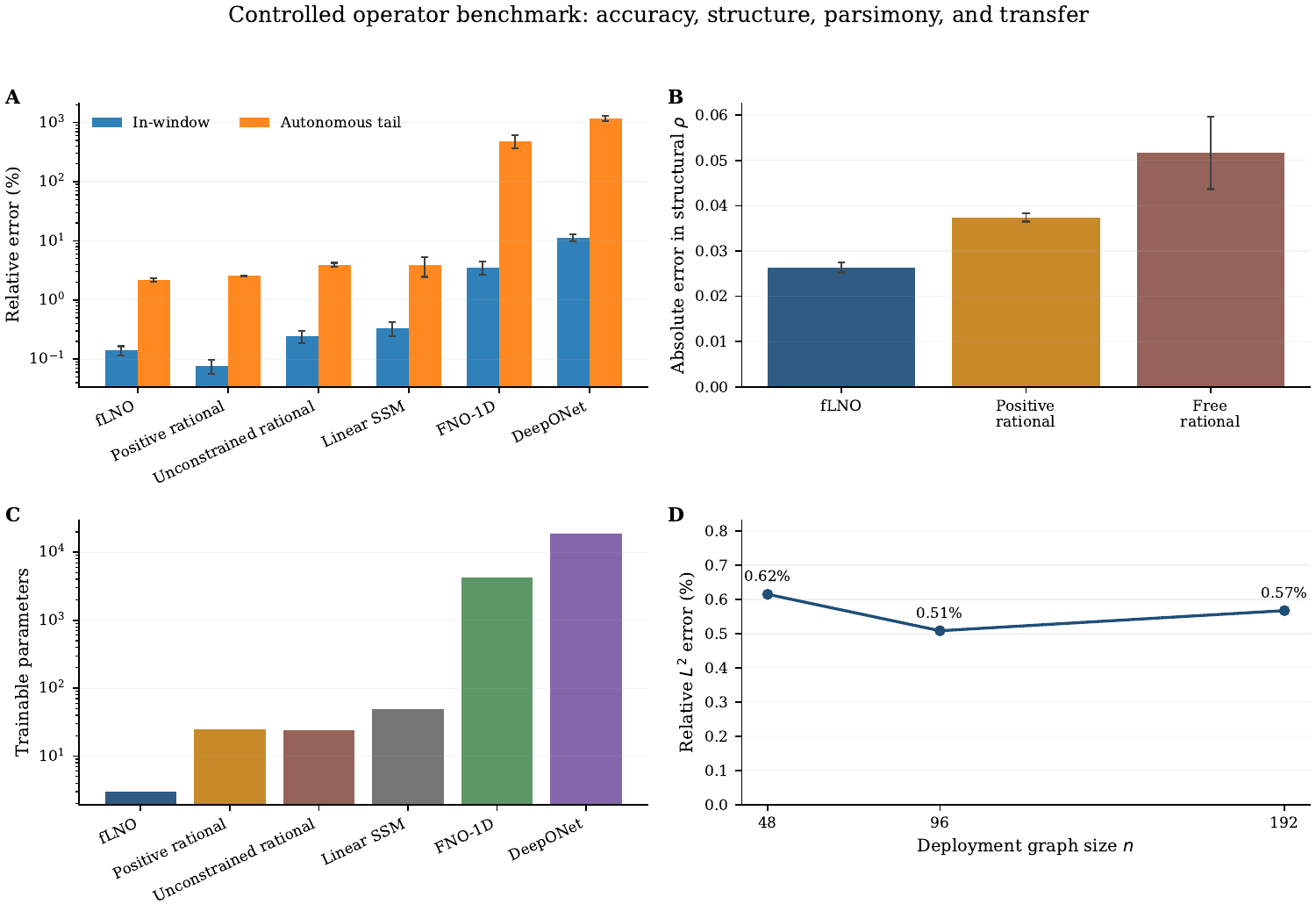}
\caption{Common-data operator benchmark. \textbf{(A)} In-window and doubled-horizon tail
relative errors (log scale). \textbf{(B)} Error in the structural branching coordinate for the three models in which
that quantity is explicit or defined by zero-frequency mass.
\textbf{(C)} Trainable parameter count. \textbf{(D)} Transfer of one learned,
size-independent fLNO parametrization across deployment graph sizes. Error
bars in panels A--B denote one standard deviation across the two main
benchmark seeds.}
\label{fig:operator-benchmark}
\end{figure}

\subsection{Near-critical accuracy versus structural identification}
\label{sec:nearcritical-benchmark}

To separate finite-horizon approximation from recovery of the stability
coordinate, we refit the fLNO and the positive rational operator at
$\rho\in\{0.70,0.90,0.97\}$ with the same 500-epoch cap. The result is not a
monotone ``winner.'' At $\rho=0.97$ the 25-parameter rational approximation
has the smaller finite-horizon tail error ($0.61\%$ versus $2.18\%$), which
is entirely consistent with Theorem~\ref{thm:rates}. Yet its explicit
branching estimate is $0.885$, an absolute error of $0.0848$, whereas the
fLNO returns $0.953$, an error of $0.0170$. Thus excellent finite-window
prediction does not by itself identify the critical coordinate.

\begin{table}[t]
\centering
\caption{Near-critical comparison after convergence-oriented training.
The rational model may be more accurate over a finite doubled horizon, while
the fLNO retains a substantially more faithful explicit branching
coordinate near the boundary.}
\label{tab:nearcritical-benchmark}
\small
\begin{tabular}{c l r r r r}
\toprule
$\rho$ & Model & In-window & Tail & $\widehat\rho$ & $|\widehat\rho-\rho|$\\
\midrule
0.70 & fLNO & 0.035\% & 0.814\% & 0.708 & 0.00845\\
0.70 & Positive rational & 0.255\% & 11.4\% & 0.810 & 0.10979\\
0.90 & fLNO & 0.101\% & 3.04\% & 0.879 & 0.02068\\
0.90 & Positive rational & 0.074\% & 5.11\% & 0.863 & 0.03692\\
0.97 & fLNO & 0.090\% & 2.18\% & 0.953 & 0.01696\\
0.97 & Positive rational & 0.038\% & 0.610\% & 0.885 & 0.08483\\
\bottomrule
\end{tabular}
\end{table}

\begin{figure}[t]
\centering
\includegraphics[width=\textwidth]{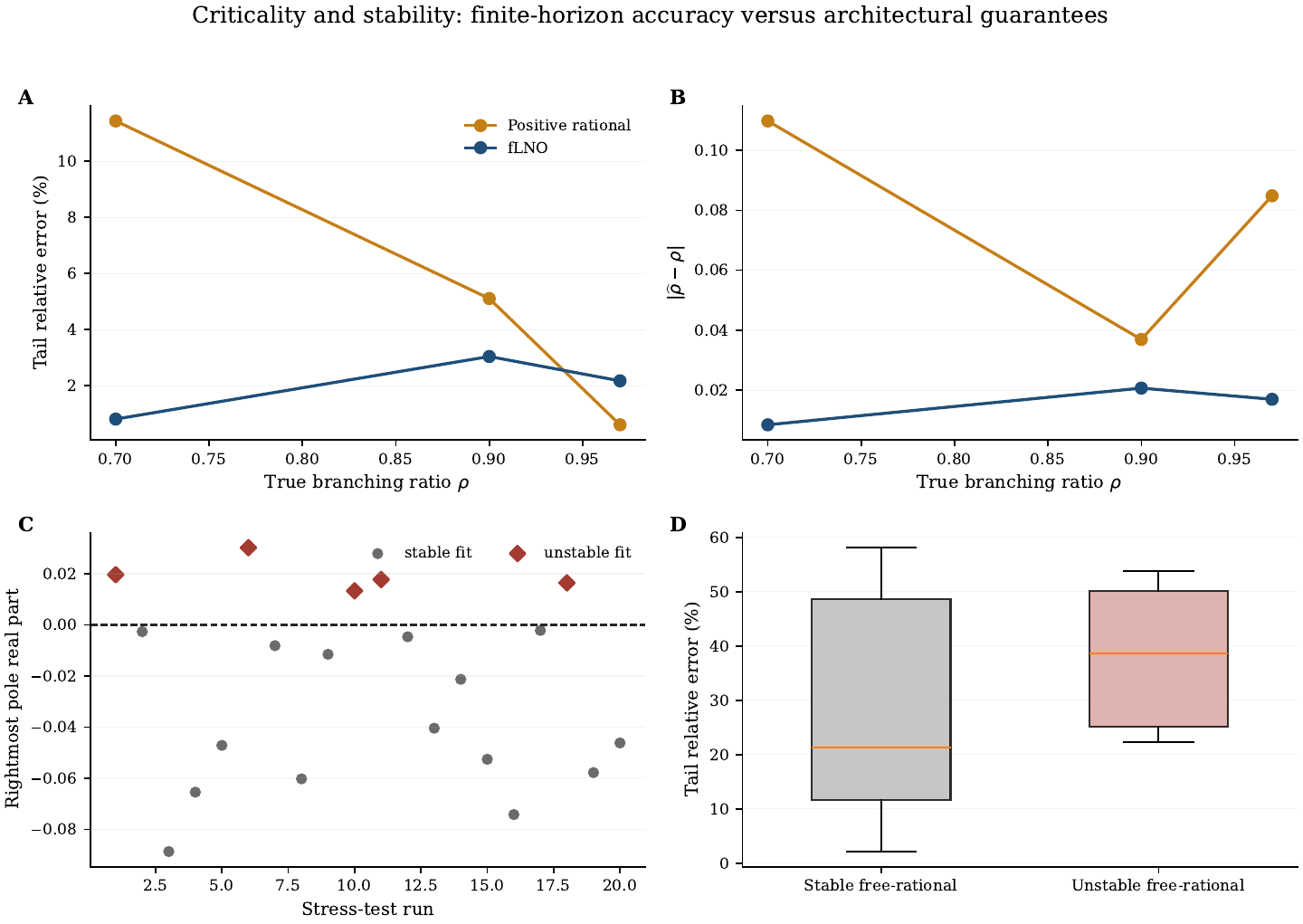}
\caption{Criticality and stability diagnostics. \textbf{(A)} Finite-horizon
tail error for fLNO and the positive rational pole--residue operator across
three true branching ratios. \textbf{(B)} Absolute error in the explicit
branching coordinate. \textbf{(C)} Rightmost-pole real parts from 20
unconstrained rational fits in the scarce/noisy stress test; diamonds mark
fits that cross the stability boundary. \textbf{(D)} Tail-error distribution
split by the resulting structural classification. Panels C--D are a
stress test, not a statement about the typical behavior of unconstrained
training.}
\label{fig:nearcritical-benchmark}
\end{figure}

\paragraph{Unconstrained stability under nominal and stress conditions.}
The nominal eight-seed audit at $\rho=0.97$ found no unstable unconstrained
rational fit: all learned rightmost poles remained in
$[-0.0145,-0.0106]$. We then deliberately stressed the same free
pole--residue class with only 12 training trajectories, Gaussian training-
response noise with standard deviation equal to $10\%$ of the training
response scale, and signed random residue initialization. Five of 20
validation-selected fits crossed the stability boundary, with the largest
rightmost-pole real part equal to $0.0302$ (Figure~\ref{fig:nearcritical-benchmark}C).
This experiment is intentionally adversarial and is not evidence that free
rational training is usually unstable. Its role is narrower: empirical
stability depends on data, initialization, and optimization, whereas the
fLNO and positive Stieltjes parametrizations remain inside their certified
stability region for every admissible weight vector.

\begin{table}[t]
\centering
\caption{Unconstrained-rational stability audit. The scarce/noisy protocol
is deliberately a stress test and is reported separately from the nominal
benchmark.}
\label{tab:stability-stress}
\small
\begin{tabular}{lrrrr}
\toprule
Protocol & runs & unstable fits & max $\Re\,\xi_\star$ & median tail error \\
\midrule
Nominal, full data & 8 & 0 & $-0.0106$ & $5.96\%$ \\
Scarce/noisy, signed random init. & 20 & 5 & $+0.0302$ & $31.9\%$ \\
\bottomrule
\end{tabular}
\end{table}

\subsection{Graph-size transfer with a fixed learned parametrization}
\label{sec:graph-benchmark}

Finally, we use the size-independent spectral law
\[
a(\ell)=\frac{\kappa}{1+\beta\ell}
\]
and learn $(\alpha,\theta,\kappa,\beta)$ on a graph of size $n=48$.
The true values $(0.65,0.16,0.42,0.75)$ are estimated as
$(0.6224,0.2081,0.4191,0.8194)$ and then frozen. Without retraining, the
same four parameters are composed with the spectra of deployment graphs of
sizes 96 and 192. Relative $L^2$ errors are respectively $0.615\%$,
$0.508\%$, and $0.567\%$ for $n=48,96,192$ (Figure~\ref{fig:operator-benchmark}D).
The error therefore stays at the same sub-percent scale when the graph size
quadruples. This experiment tests transfer of learned physical parameters
through a size-independent spectral parametrization; it is not an empirical
estimate of the operator norm appearing in Theorem~\ref{thm:transfer}.

\FloatBarrier
\section{Application I: a single-source relocated Chilean aftershock cohort}
\label{sec:seismic}

The seismic application is rebuilt from a single regional source rather than
from heterogeneous sequence files, retaining dependent events because aftershock clustering is the signal of interest for the Hawkes--Volterra model \citep{ogata1988}. We use the relocated Centro Sismol\'ogico
Nacional (CSN) catalog of \citet{potin2024}, archived as Zenodo
\url{https://doi.org/10.5281/zenodo.13146436}. The release contains
118,004 relocated events from 1982 to mid-2020 with origin time,
hypocentral coordinates, depth, location RMS, magnitude, and magnitude type.
The primary cohort comprises Tocopilla 2007, Maule 2010, Iquique 2014, and
Illapel 2015. These four sequences are analyzed with the same extraction,
completeness, fitting, and spatial-zoning code. This design removes a major
source of between-sequence heterogeneity while retaining event-level spatial
information. The more extensive catalog construction of \citet{candia2026}
provides a complementary benchmark for duplicate resolution, magnitude
homogenization, tectonic classification, and completeness analysis. We do
not decluster the Potin catalog: dependent aftershocks are the
self-excitation signal of interest.

\subsection{Cohort construction and completeness}

For each sequence, the clock is anchored to the reviewed USGS ComCat
mainshock origin and events are selected within 350 km and 365 days after
the mainshock. A one-minute exclusion around the external anchor prevents a
catalog representation of the mainshock from being counted twice; the
mainshock is then inserted once into day zero of the history. This convention
is particularly useful for Illapel, whose $M_w8.3$ mainshock is not present
as such in the relocated event file although its aftershock sequence is.
All four extracted cohorts have unique event-level latitude--longitude pairs
for essentially every retained event and non-degenerate depth distributions.

Magnitude is used only to define an occurrence threshold, not as a Hawkes
mark. Because the source file contains several CSN magnitude types, no claim
of event-by-event $M_w$ homogeneity is made. For each sequence we estimate a
conservative completeness threshold as the largest of (i) maximum curvature
plus 0.2 magnitude units, (ii) the first 90\% Gutenberg--Richter
goodness-of-fit threshold, and (iii) a local $b$-value stability criterion,
rounded upward to 0.1 magnitude units. The resulting thresholds are
$M_c=3.3$ (Tocopilla), $3.8$ (Maule), $3.1$ (Iquique), and $3.7$
(Illapel), leaving 909, 1645, 1835, and 1418 aftershocks, respectively.
Sensitivity analyses repeat the fit over 180, 365, and 730 day observation
windows and at $M_c$, $M_c+0.25$, and $M_c+0.50$.

\subsection{Scalar branching inference}

Daily counts are fitted with the finite-window Hawkes--Volterra model
\[
\lambda_t=\mu+\frac{\rho}{G_L}\sum_{\ell=1}^{\min(L,t)}
 m_\ell Y_{t-\ell},\qquad
m_\ell=(\ell+c)^{-p}e^{-\theta\ell},\qquad
G_L=\sum_{\ell=1}^{L}m_\ell,
\]
with $L=180$ days. The branching ratio $\rho$ is parameterized directly;
therefore it remains the primary stability diagnostic even when the nuisance
shape parameters $(p,c,\theta)$ are weakly identified. We compare the
tempered model with the nested untempered model $\theta=0$ using
$\Delta\mathrm{AIC}=\mathrm{AIC}_{\rm temp}-\mathrm{AIC}_{\rm untemp}$.
Profile likelihood supplies the primary 95\% interval for $\rho$. A
300-replicate parametric bootstrap is also run from the fitted model. We
report its bias-corrected point estimate and basic bootstrap interval only as
a finite-sample sensitivity diagnostic; no nominal coverage claim is made.

\begin{table}[t]
\centering
\caption{Scalar branching diagnostics from the single-source relocated
Chilean cohort. Profile intervals are the primary uncertainty statement.
The basic parametric-bootstrap interval ($B=300$) is a finite-sample
sensitivity diagnostic, not a calibrated nominal-coverage interval.
$\Delta$AIC compares the tempered and untempered kernels; negative values
favor tempering.}
\label{tab:seismic}
\scriptsize
\begin{tabular}{lrrrrrc}
\toprule
sequence & $M_c$ & $N$ & $\widehat\rho$ [profile] &
$\widetilde\rho$ [bootstrap] & $\widehat p$ & $\Delta$AIC \\
\midrule
Tocopilla & 3.3 &  909 & $0.781\,[0.644,0.948]$ & $0.838\,[0.692,1.071]$ & 1.22 & $+1.2$ \\
Maule     & 3.8 & 1645 & $0.787\,[0.738,0.839]$ & $0.792\,[0.714,0.881]$ & 0.80 & $-15.9$ \\
Iquique   & 3.1 & 1835 & $0.652\,[0.611,0.695]$ & $0.660\,[0.598,0.727]$ & 2.28 & $-1.7$ \\
Illapel   & 3.7 & 1418 & $0.722\,[0.648,0.803]$ & $0.745\,[0.653,0.840]$ & 1.72 & $+2.0$ \\
\bottomrule
\end{tabular}
\end{table}

\begin{figure}[t]
\centering
\includegraphics[width=\textwidth]{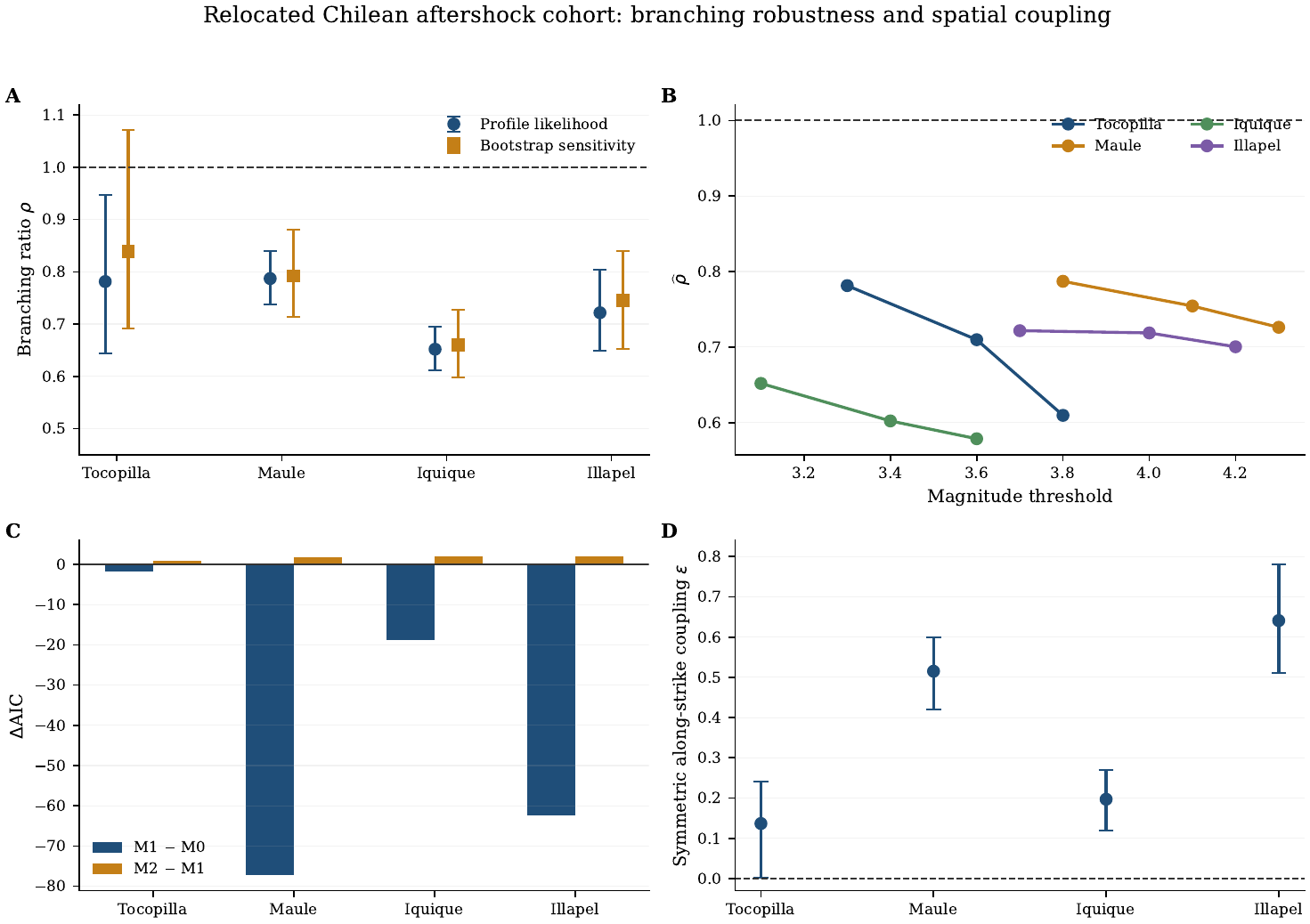}
\caption{Single-source Chilean seismic analysis. \textbf{(A)} Scalar
branching ratio from profile likelihood and from the 300-replicate
parametric-bootstrap sensitivity analysis; the dashed line is the critical
value $\rho=1$. \textbf{(B)} Dependence of $\widehat\rho$ on increasingly
conservative magnitude thresholds for the 365-day window. \textbf{(C)} AIC
ladder for the ordered along-strike spatial operator: M0 has no inter-zone
mixing, M1 has symmetric nearest-neighbor mixing, and M2 permits separate
northward and southward strengths. \textbf{(D)} Estimated symmetric M1
coupling with profile intervals.}
\label{fig:seismic-chile}
\end{figure}

The profile likelihood places all four primary estimates below the critical
boundary. Across the full $(T,M_c)$ sensitivity grid, the fitted branching
ratios remain below one: $0.46$--$0.85$ for Tocopilla, $0.73$--$0.87$ for
Maule, $0.54$--$0.68$ for Iquique, and $0.59$--$0.74$ for Illapel. The
finite-sample bootstrap is more cautious: its basic interval crosses one for
Tocopilla, so that sequence is treated as bootstrap-sensitive rather than as
robustly certified. The other three basic bootstrap intervals remain
subcritical. The nuisance memory shape is not interpreted mechanistically:
several fits approach bounds in $p$, $c$, or $\theta$, and the AIC comparison
shows that only Maule provides strong evidence for a finite tempering term,
with modest evidence for Iquique and essentially equivalent fits for
Tocopilla and Illapel. In particular, no empirical claim of an asymptotically
untempered physical kernel is made.

\subsection{Ordered along-strike spatial operator}

For the spatial stage, hypocenters are projected to local Cartesian
coordinates, the dominant rupture-aligned direction is estimated by principal
components, and events are divided into three tertiles along the oriented
axis. The sign of the axis is fixed to point northward, so zone labels are
reproducible and physically ordered. Figure~\ref{fig:seismic-zones} shows the
resulting zones. This construction replaces arbitrary $k$-means labels by a
one-dimensional along-strike coordinate.

The scalar memory shape is held fixed at its sequence-specific MLE. We then
compare three spatial models whose coupling matrices have Perron radius one. M0 is independent
across zones. M1 moves a total fraction $\varepsilon$ of excitation to the
two adjacent zones symmetrically. M2 replaces this by separate northward and
southward fractions. Boundary leakage is returned to the source zone, so all
three coupling matrices are column stochastic and have Perron root one;
therefore the fitted scalar $\rho$ is also the network branching radius.

\begin{table}[t]
\centering
\caption{Ordered along-strike spatial ladder. The $\varepsilon$ interval is
a one-dimensional profile interval under M1. Negative
$\Delta\mathrm{AIC}_{10}$ favors symmetric spatial mixing over independent
zones; positive $\Delta\mathrm{AIC}_{21}$ means that an additional directional
north--south asymmetry is not supported.}
\label{tab:seismic-spatial}
\small
\begin{tabular}{lrrrr}
\toprule
sequence & $\widehat\rho_{\rm M1}$ & $\widehat\varepsilon$ [profile] &
$\Delta\mathrm{AIC}_{10}$ & $\Delta\mathrm{AIC}_{21}$ \\
\midrule
Tocopilla & 0.655 & $0.137\,[0.001,0.240]$ & $-1.8$  & $+0.8$ \\
Maule     & 0.767 & $0.515\,[0.420,0.600]$ & $-77.1$ & $+1.7$ \\
Iquique   & 0.597 & $0.197\,[0.120,0.270]$ & $-18.9$ & $+1.9$ \\
Illapel   & 0.695 & $0.641\,[0.510,0.780]$ & $-62.5$ & $+2.0$ \\
\bottomrule
\end{tabular}
\end{table}

\begin{figure}[t]
\centering
\includegraphics[width=0.94\textwidth]{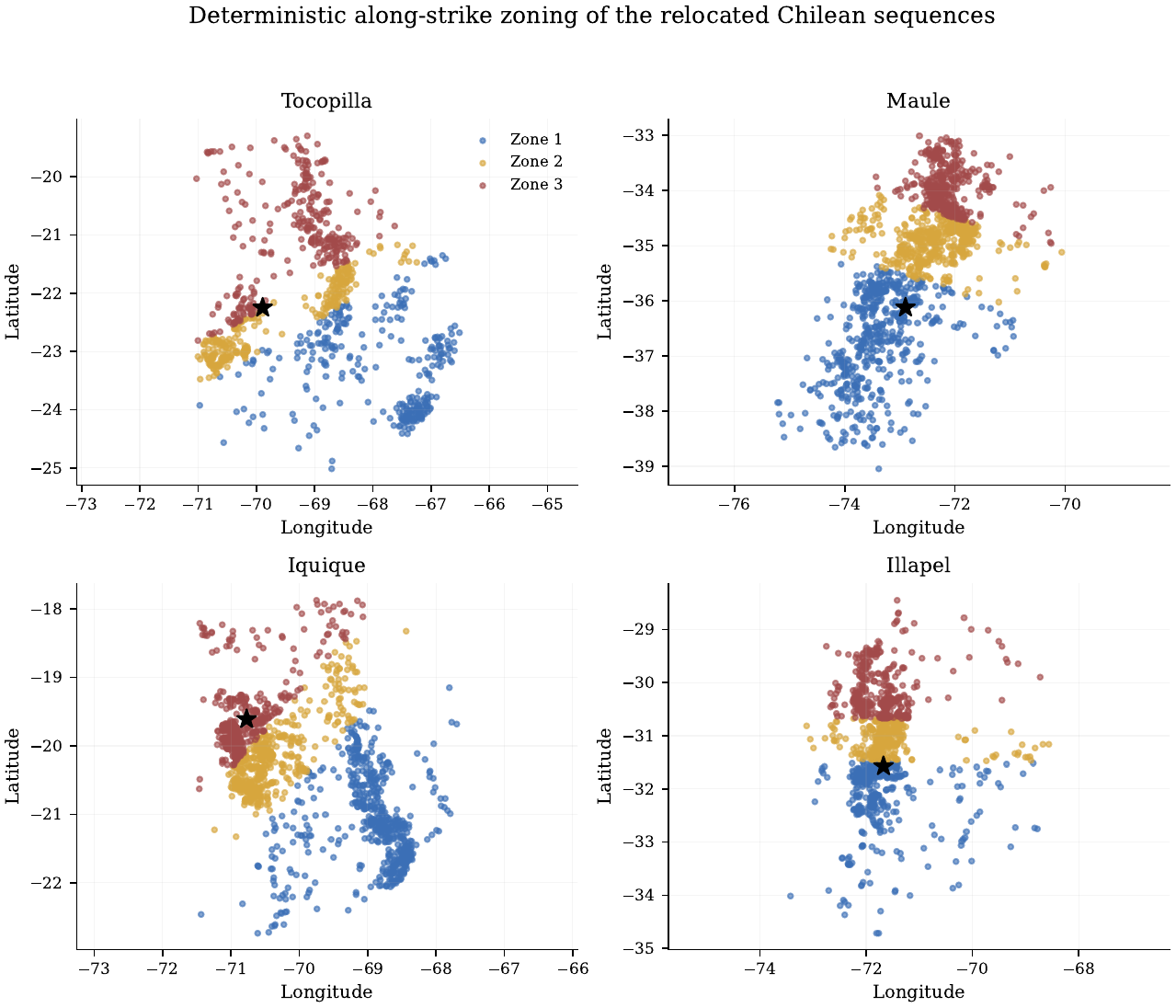}
\caption{Ordered along-strike zones for the four relocated Chilean
sequences. Colors denote tertiles of the principal rupture-aligned
coordinate; stars are the external reviewed mainshock anchors. The zoning is
deterministic and is used only after the sequence-specific completeness
filter.}
\label{fig:seismic-zones}
\end{figure}

The spatial result is more informative than the earlier unrestricted
three-zone experiment. Symmetric inter-zone mixing is strongly supported for
Maule, Iquique, and Illapel, with profile intervals for $\varepsilon$ bounded
away from zero. Tocopilla is only weakly informative: its AIC improvement is
small and the profile reaches the near-zero edge of the grid. In every
sequence, M2 fails to improve AIC despite its extra directional parameter.
Thus the present data support an identifiable low-dimensional spatial
coupling component in three sequences, but not a resolved north--south
asymmetry. This distinction is useful for the operator-learning narrative:
the certified matrix fLNO can identify network coupling when the spatial
parameterization is physically ordered and parsimonious, while refusing an
unsupported directional degree of freedom.

\FloatBarrier
\section{Application II: the renewal equation as a certified layer}
\label{sec:epi}

The epidemic renewal equation is the intensity block of
\eqref{eq:volterra} at daily resolution, and the application inverts two
roles instructively. Normalizing the learned kernel to unit mass makes
the branching diagnostic \emph{literally} the effective reproduction
number, $\rho_k=R_k$ window by window. And since epidemics do cross
criticality, the certificate is not imposed on estimation; its role is
the control one --- margin to criticality per window, and the
intervention strength guaranteeing a prescribed margin. Finally, the
rational block of Definition~\ref{def:flno} earns its keep here:
generation-interval kernels are humped, outside the completely monotone
atom class, and we use the block's minimal nonnegative humped element
$w_\ell=e^{-b\ell}-e^{-(b+d)\ell}$. The two applications thus exercise
complementary stages of the same architecture.

\subsection{Scalar stage: Chile}

Daily national cases from the OWID-hosted historical JHU series \citep{dong2020} (2020-03-15 to 2022-03-01, $717$ days) are
fitted with
$I_t\sim\mathrm{NegBin}(\delta_{\mathrm{dow}(t)}R_{k(t)}
\sum_{\ell\le21}m_\ell I_{t-\ell},\kappa)$ over $24$ four-week windows
($33$ parameters, joint ML; exact one-dimensional conditional profiles
per window), against the Cori estimator \citep{cori2013} with the
literature gamma kernel on the same windows
(Figure~\ref{fig:epi-chile}). Findings: \emph{(i)} the trajectory
validates against the country's known history, with $92\%$
growth/decline classification agreement with the baseline; \emph{(ii)}
learning the kernel is decisively supported,
$\Delta\mathrm{AIC}=-62.7$ against the same machine with the kernel
frozen, the learned effective kernel contracting to mean $3.3$ d ---
reported as an effective reporting-resolution kernel. The critical normalization remains $R=1$ by construction, but the estimated $R_t$ trajectory and hence empirical classification do depend on the learned kernel; \emph{(iii)} the uncertainty comparison is substantial: with $\widehat\kappa\approx16$, the fixed-kernel Cori intervals are substantially narrower than the negative-binomial profile intervals under the present model. Across the 24 windows, the median ratio of profile-interval width to Cori-interval width is $9.4$ (mean $10.8$, range $4.3$--$29.0$). This comparison is descriptive and is not itself a coverage proof; only $8$ of $24$ windows decisively exclude $R=1$.

\begin{figure}[t]
\centering
\includegraphics[width=\textwidth]{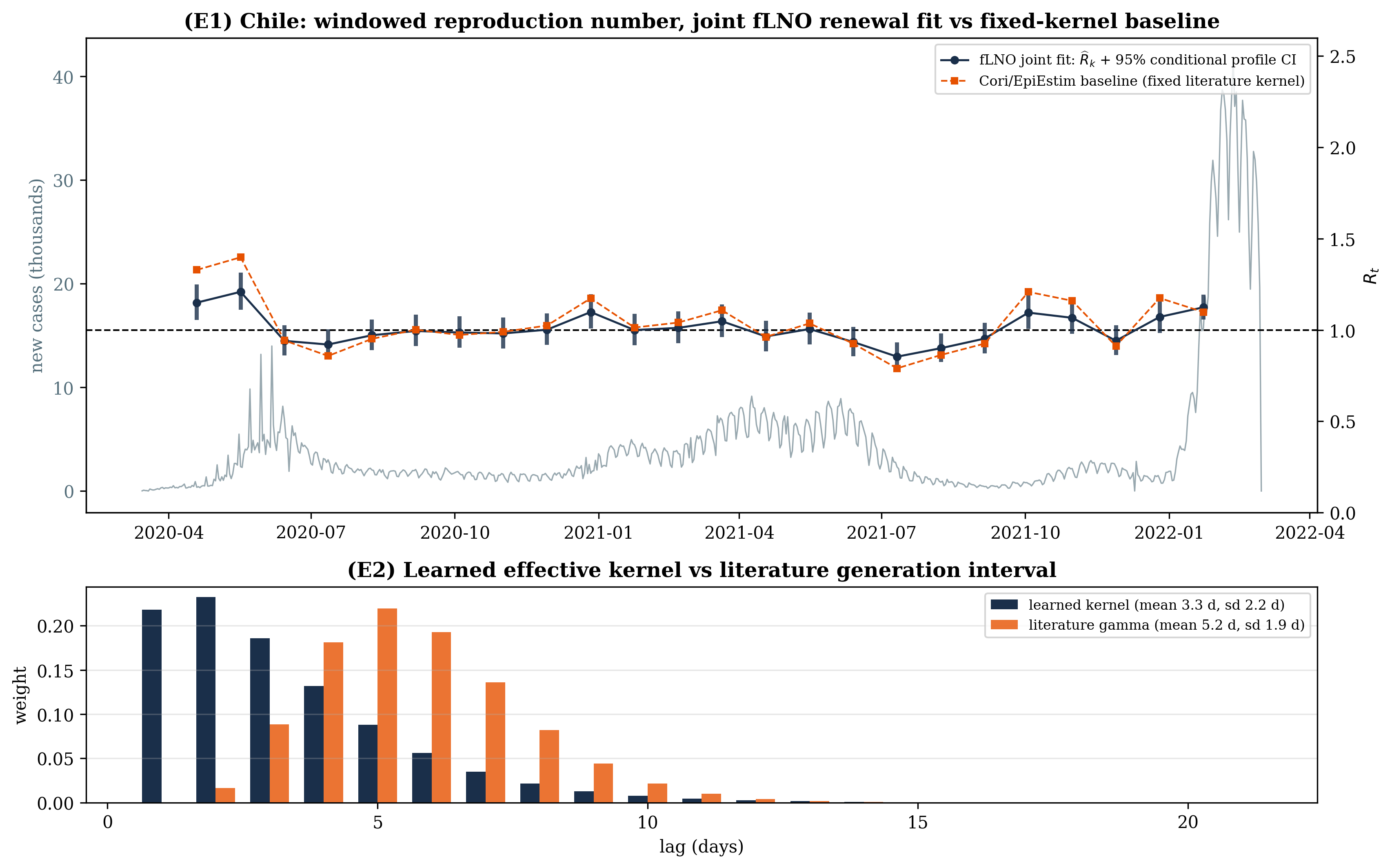}
\caption{Chile, scalar stage. \textbf{(E1)} Windowed $R_t$ from the joint
fit (dark, $95\%$ conditional profiles) against the fixed-kernel Cori
baseline (light), over the incidence backdrop. \textbf{(E2)} Learned
effective kernel against the literature generation interval.}
\label{fig:epi-chile}
\end{figure}

\subsection{Network stage: 21 Italian regions}

Regional daily incidence from the Italian Protezione Civile \citep{dpc2020} (2020-09-01 to 2022-03-01;
$547$ days, $12.4$ million cases, $18$ windows) is fitted with
$I_{i,t}\sim\mathrm{NegBin}(\delta_{\mathrm{dow}(t)}R_{k(t)}
(A\Lambda_t)_i,\kappa)$, the coupling ladder Perron-normalized by
construction ($\rho_{\rm net}(k)=R_k$ exactly): M0 diagonal; M1
destination-scaled mean field $A=(1-\varepsilon)I+\varepsilon
p\mathbf 1^{\!\top}$ with frozen 2020 resident-population weights from ISTAT \citep{istat2021}; M2 gravity
$A\propto(1-\varepsilon)I+\varepsilon\,\mathrm{diag}(p)\widetilde
W(\ell)$, $\widetilde W_{ij}\propto\mathrm{pop}_je^{-d_{ij}/\ell}$. One
modelling point deserves record because it changed the answer: with raw
counts the importation term must scale with destination population; a
plain row-stochastic mixing loses to the mean field.

\begin{table}[h]
\centering
\caption{Italy, 21 regions: the certified coupling ladder.}
\label{tab:italy}
\small
\begin{tabular}{lccll}
\toprule
model & params & AIC & $\Delta$AIC & identified structure \\
\midrule
M0 diagonal   & 27 & 132\,831.9 & ---      & --- \\
M1 mean field & 28 & 132\,701.2 & $-130.7$ &
$\widehat\varepsilon=0.037$, CI $[0.033,0.040]$ \\
M2 gravity    & 29 & 132\,685.0 & $-16.2$  &
$\widehat\ell=403$ km, CI $[234,894]$ \\
\bottomrule
\end{tabular}
\end{table}

\begin{figure}[t]
\centering
\includegraphics[width=\textwidth]{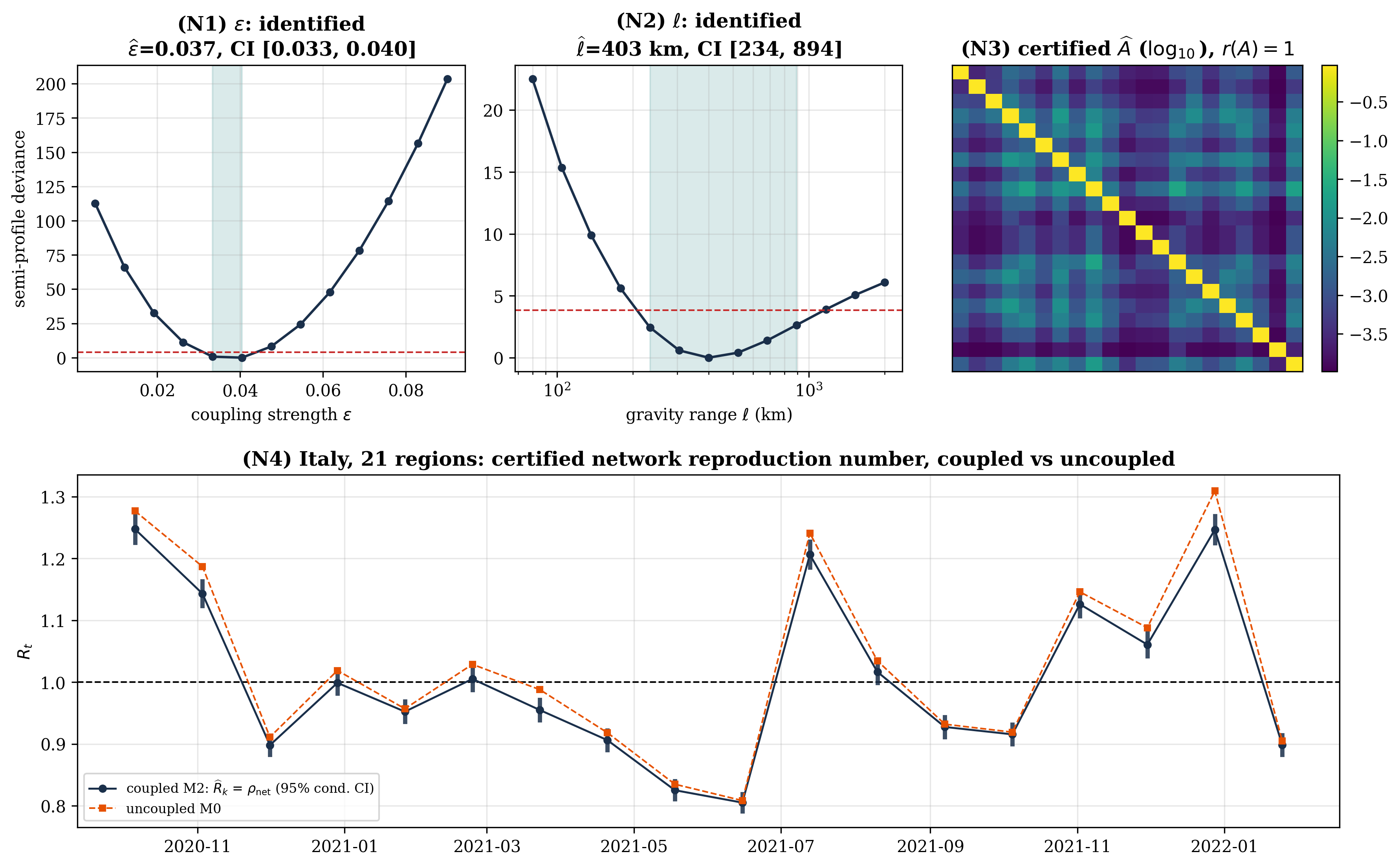}
\caption{Italy, network stage. \textbf{(N1)}--\textbf{(N2)} Semi-profile
deviances for the coupling strength and the gravity range.
\textbf{(N3)} The certified reproduction matrix, $r(\widehat A)=1$ by
construction. \textbf{(N4)} Certified network reproduction number with
$95\%$ conditional intervals, coupled against uncoupled.}
\label{fig:epi-italy}
\end{figure}

Results (Table~\ref{tab:italy}, Figure~\ref{fig:epi-italy}): coupling is
identifiable at $q=21$ with ten-percent relative precision
($\widehat\varepsilon=0.037$, CI $[0.033,0.040]$; deviance $\approx113$
at $\varepsilon\to0$), and the spatial range is identified as well
($\widehat\ell=403$ km, CI $[234,894]$). Under the fitted model, omitting spatial coupling yields larger local reproduction estimates, with the largest differences at wave onsets (Delta: $1.20$ vs $1.24$; Omicron: $1.24$ vs $1.31$). AIC is supplemented by an executed rolling-origin test: fitting windows $1{:}k$ and scoring window $k+1$ for $k=10,\ldots,17$ gives a mean M2--M0 negative-binomial log-score gain of $+0.039$ nats per region-day (median $+0.025$), with M2 improving on M0 in $6$ of $8$ origins. A legacy single split gives a consistent gain of $+0.041$ nats per region-day. These predictive checks support information in the low-dimensional coupling, but neither AIC nor forecasting gain is interpreted as causal evidence for physical importation.

\begin{figure}[t]
\centering
\includegraphics[width=0.90\textwidth]{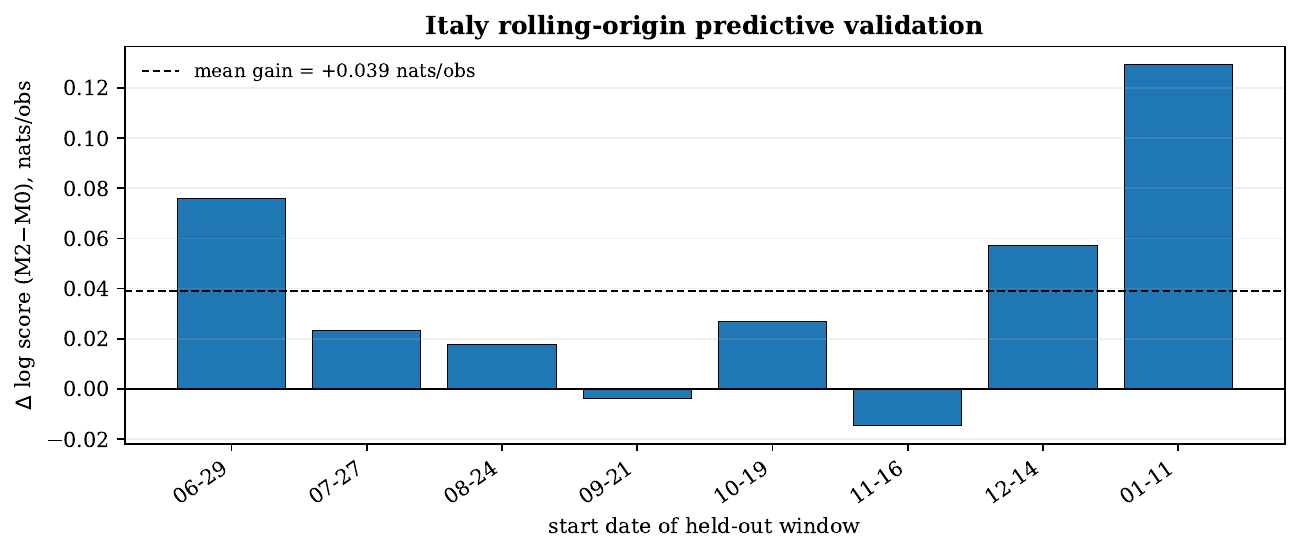}
\caption{Italy rolling-origin predictive validation. At each origin, M0 and M2 are refitted using all preceding 28-day windows and the next window is scored by the negative-binomial log score under one-step $R$ persistence. Positive values favor the gravity-coupled M2 model. M2 improves on M0 in six of eight origins; the dashed line marks the mean gain of $+0.039$ nats per region-day.}
\label{fig:epi-italy-rolling}
\end{figure}

\subsection{The structural contrast}

The two applications state the identifiability message empirically: the
same certified estimator, run honestly at two information scales, gives
at $q=3$ with $\sim6\times10^3$ events an identified branching radius
with a partially flat coupling direction reported as such, and at $q=21$
with $\sim1.2\times10^7$ cases, coupling strength and spatial range
identified to ten percent. The certified parametrization is what makes
the runs comparable: in both, the stability object is pinned to
$r(A)=1$ by construction rather than estimated along an unstable
direction.

\section{Discussion and limitations}
\label{sec:discussion}

The main mathematical conclusion is narrower and stronger than a generic
claim that neural operators can perform stability analysis. The fLNO
contains the exact fractional Volterra resolvent on the commuting class,
and its trainable gains can be restricted to the exact known stability
region. This complements equation-free neural-operator stability methods
\citep{fabiani2026}: those methods interrogate a learned time-stepper,
whereas the present certificate is an architectural property tied to a
specific hereditary branching threshold.

The rational expressivity result also needs careful interpretation. On a
fixed frequency window, positive rational approximants converge
geometrically and can be extremely accurate. The obstruction concerns
exact fractional branch singularities and horizon-uniform critical
asymptotics, not ordinary finite-window interpolation. Moreover, a general
signed rational transfer function can have a multiple critical zero and
therefore polynomial terms; the stronger positive-neutral-limit statement
holds only for the positive Stieltjes rational class.

The transfer theorem is an operator statement. Empirical accuracy on a
finite collection of training trajectories is not an operator norm. The
benchmark therefore separates what is demonstrated empirically from what is
proved analytically. On the common-data task, a positive rational
pole--residue operator can match or exceed fLNO accuracy over a finite
window, and at $\rho=0.97$ it even has the smaller doubled-horizon tail
error. This is not a failure of the expressivity theorem: the obstruction is
the exact fractional branch singularity and horizon-uniform asymptotics,
not finite-window approximation. Conversely, the fLNO identifies the
branching coordinate more faithfully and does so with three parameters.
The FNO-1D and DeepONet comparisons are benchmark-specific rather than a
general ranking. The main benchmark reports two independent seeds, so its
small error differences should be interpreted descriptively rather than as
population-level performance claims. The rational/SSM baselines are canonical
pole--residue/state-space realizations rather than official LNO or S4 code.
The graph experiment tests transfer of size-independent physical parameters
composed with the deployment graph's known spectral structure; it does not
estimate the operator norm in Theorem~\ref{thm:transfer}.

The seismic analysis is intentionally conservative. All four sequences are
now derived from one relocated CSN catalog \citep{potin2024}. The
reproducibility package freezes the source DOI and SHA-256 identifier, the
retrieval/extraction code, mainshock anchors, completeness diagnostics, and
fitted outputs; the event-level source table itself is not republished because
the current source record does not display an explicit redistribution license. Magnitude is
used only as a sequence-specific occurrence threshold because the source
catalog contains several magnitude types. The 300-replicate parametric
bootstrap is a finite-sample sensitivity analysis, not a validated 95\%
calibration; an outer simulation study would be required before making
nominal-coverage claims. The catalog of \citet{candia2026} remains a useful
independent route for future sensitivity analyses, but any Hawkes analysis
must use a duplicate-resolved, non-declustered event table because
declustering removes the self-excitation signal.

The epidemic results are likewise associational. Learned generation or
reporting kernels alter the estimated $R_t$ trajectory even though the
critical normalization is $R=1$. The Italian spatial coupling matrix may
absorb common national forcing, reporting synchronization, or policy
effects in addition to physical importation; no causal mobility
interpretation is made without mobility covariates.

Finally, exactness currently concerns linear excitation. Nonlinear or
saturating hereditary operators are the natural setting in which deeper
nonlinear fLNO compositions would provide approximation power beyond the
exact linear resolvent. The seismic and epidemic studies use the certified
memory/coupling architecture for structured statistical estimation; they are
not end-to-end black-box neural-operator benchmarks and should not be
interpreted as such.

\paragraph{Data and code availability.}
Code, frozen numerical outputs, analysis scripts, provenance records, and
reproducibility documentation are publicly available at
\href{https://github.com/mauricio-herrera/fractional-laplace-neural-operators}{project GitHub repository}.
The seismic application uses the relocated Chilean CSN catalog of
\citet{potin2024}, archived at
\href{https://doi.org/10.5281/zenodo.13146436}{Zenodo 10.5281/zenodo.13146436}. Because the current source
record does not display an explicit redistribution license, the event-level
catalog is not republished; the repository instead provides an automated
retrieval script and verifies the exact source by SHA-256 before reconstruction.
The epidemic analyses use the historical JHU/OWID Chile series and the official
Italian Protezione Civile regional series under their original data terms, with
frozen provenance and computational population weights. The repository contains
the primary common-data benchmark, near-critical sweep, nominal and scarce/noisy
stability audits, graph-size transfer experiment, seismic reconstruction, and
Chile/Italy epidemic workflows including the executed eight-origin rolling
validation. The previously published Volterra foundation and its archived
reproducibility package are available at
\href{https://doi.org/10.5281/zenodo.21779954}{Zenodo 10.5281/zenodo.21779954}. The version-specific v1.0.0 archive of the present repository is available at
\href{https://doi.org/10.5281/zenodo.22837154}{Zenodo 10.5281/zenodo.22837154}.

\appendix

\section{Proof of Theorem \ref{thm:rates}}
\label{app:rates}

Throughout, $c_\alpha=\sin(\pi\alpha)/\pi$.

\begin{lemma}[Bernstein representation]
\label{lem:bern}
For $\xi>-\theta$,
$(\xi+\theta)^{\alpha-1}=c_\alpha\int_\theta^\infty
(r-\theta)^{\alpha-1}(\xi+r)^{-1}dr$.
\end{lemma}

\begin{proof}
For $x>0$, $a\in(0,1)$, substituting $s=xt$:
$\int_0^\infty s^{a-1}(s+x)^{-1}ds=x^{a-1}B(a,1-a)
=\pi x^{a-1}/\sin\pi a$. Take $a=\alpha$, $x=\xi+\theta$, $s=r-\theta$.
\end{proof}

\begin{lemma}[Jacobi representation]
\label{lem:jacobi}
For $\xi\ge0$,
\[
\ghat(\xi)=\int_0^1\frac{d\sigma(v)}{\theta+\xi v},\qquad
d\sigma(v)=c_\alpha\theta^\alpha v^{-\alpha}(1-v)^{\alpha-1}dv,\qquad
\sigma\big((0,1)\big)=\theta^\alpha .
\]
\end{lemma}

\begin{proof}
Substitute $r=\theta/v$ in Lemma~\ref{lem:bern}:
$dr=-\theta v^{-2}dv$,
$(r-\theta)^{\alpha-1}=\theta^{\alpha-1}(1-v)^{\alpha-1}v^{1-\alpha}$,
$(\xi+r)^{-1}=v/(\theta+\xi v)$; collecting powers of $v$ gives the
identity, and the mass follows from
$B(1-\alpha,\alpha)=1/c_\alpha$; at $\xi=0$ the formula returns
$\theta^{\alpha-1}$ as it must. The identity was additionally certified
numerically to twelve digits.
\end{proof}

The weight is a Jacobi weight with both exponents in $(-1,0)$, so
$n$-point Gauss quadrature for $\sigma$ exists with nodes in $(0,1)$ and
positive weights, exact on polynomials of degree $2n-1$
\citep{gautschi}.

\begin{lemma}[Gauss error for analytic integrands]
\label{lem:gauss}
Let $\mu$ be a finite positive measure on $[a,b]$ and $f$ analytic on the
closed Bernstein ellipse $\overline{E_\rho}$ of $[a,b]$ with $|f|\le M$,
$\rho>1$. Then the $n$-point Gauss rule satisfies
$|\int f\,d\mu-G_n[f]|\le 4M\mu([a,b])\rho^{-2n}/(1-\rho^{-1})$.
\end{lemma}

\begin{proof}
Map to $[-1,1]$ and expand in Chebyshev polynomials,
$f=\sum'a_kT_k$ with $|a_k|\le2M\rho^{-k}$
\citep[Thm.~8.1]{trefethen}. Gauss exactness on degree $2n-1$ removes the
first $2n$ terms; for each remaining $k$, $|\int T_kd\mu|\le\mu$ and
$|G_n[T_k]|\le\mu$ by positivity of the weights and
$\|T_k\|_\infty\le1$. Summing the tail gives the bound.
\end{proof}

\subsection*{Part (i): compact windows}

Let $(v_j,w_j)_{j\le N}$ be the Gauss rule for $\sigma$ and
$S_N(\xi)=\sum_jw_j/(\theta+\xi v_j)=\sum_j\beta_j/(\xi+r_j)$ with
$r_j=\theta/v_j>\theta$, $\beta_j=w_j/v_j>0$, so
$S_N\in\mathcal S_N^{+}$. Fix $\xi\in[0,\Xi]$ and apply
Lemma~\ref{lem:gauss} to $f_\xi(v)=(\theta+\xi v)^{-1}$ with
$\rho'=1+\tau+\sqrt{\tau(2+\tau)}$, whose leftmost real point in the
$[-1,1]$ coordinate is $-(1+\tau)$, i.e.\ $v=-\tau/2$. The only
singularity of $f_\xi$ is $v^\ast=-\theta/\xi\le-\tau$, outside
$\overline{E_{\rho'}}$; on the ellipse,
$\Re(\theta+\xi v)\ge\theta-\xi\tau/2\ge\theta/2$ uniformly in
$\xi\in[0,\Xi]$, so $M\le2/\theta$; with mass $\theta^\alpha$ the bound
of part (i) follows. Numerically the construction tracks the sharp
parameter $\rho_0$ per pole to three--four digits across windows
$\Xi\in\{20,100,500\}$.

\subsection*{Part (ii): half-line}

\emph{Reduction.} $u=\theta/(\xi+\theta)$ is a bijection
$[0,\infty]\to[0,1]$ with
$\ghat(\xi)=\theta^{\alpha-1}u^{1-\alpha}$; a M\"obius substitution maps
degree-$N$ rationals bijectively to degree-$N$ rationals with equal sup
norms (competitors with a $u$-pole in $[0,1]$ have infinite error and are
never optimal; both $\ghat$ and any bounded competitor extend
continuously to the endpoint). Hence
$E_N^{\rm free}([0,\infty))=\theta^{\alpha-1}E_N(u^{1-\alpha};[0,1])$,
and Stahl's asymptotics \citep{stahl2003} give the lower bound for all
rationals, a fortiori for $\mathcal S_N^{+}$.

\emph{Dyadic positive construction.} In the shifted variable
$s=r-\theta$, $\ghat(\xi)=c_\alpha\int_0^\infty
s^{\alpha-1}\varphi_\xi(s)\,ds$, $\varphi_\xi(s)=(\xi+\theta+s)^{-1}$.
Fix $n\ge2$, $m\ge1$, $R=\theta2^m$; decompose into $J_0=(0,\theta]$,
$J_j=(\theta2^{j-1},\theta2^j]$, and the tail. The tail is at most
$c_\alpha\theta^{\alpha-1}2^{-(1-\alpha)m}/(1-\alpha)$, uniformly in
$\xi\ge0$. On $J_0$, the $n$-point Gauss rule of
$c_\alpha s^{\alpha-1}ds$ applied to $\varphi_\xi$: the singularity
$-(\xi+\theta)\le-\theta$ lies outside the ellipse $E_3$ of $[0,\theta]$
(leftmost point $-\theta/3$), $\Re(\xi+\theta+s)\ge\theta/2$ there, and
Lemma~\ref{lem:gauss} gives error
$\le(12c_\alpha/\alpha)\theta^{\alpha-1}3^{-2n}$, with positive residues
and poles in $(\theta,2\theta)$. On $J_j$, $n$-point Gauss--Legendre
applied to the folded integrand $c_\alpha s^{\alpha-1}\varphi_\xi$: on
$\overline{E_3}$ of $[a,2a]$, $a=\theta2^{j-1}$, the leftmost point is
$2a/3>0$, so $|s|^{\alpha-1}\le(2a/3)^{\alpha-1}$ and
$|\xi+\theta+s|\ge2a/3$; per-block error
$\le6c_\alpha(3/2)^{2-\alpha}a^{\alpha-1}3^{-2n}$, and
$\sum_{j\ge1}a_j^{\alpha-1}=\theta^{\alpha-1}/(1-2^{\alpha-1})$ keeps the
total at $C_\alpha\theta^{\alpha-1}3^{-2n}$ uniformly in $m$. With
$N=(m+1)n$ and
$m=\lceil\sqrt{2(\log3)N/((1-\alpha)\log2)}\rceil$ the two exponents
balance and $E_N^{+}\le C_Be^{-c_1\sqrt N}$,
$c_1=\sqrt{2(\log3)(\log2)(1-\alpha)}$. The implemented construction
measures $e^{-0.89\sqrt N}$ at $\alpha=0.7$, inside the theoretical
envelope. \qed

\section{Remaining proofs}
\label{app:aux}

\subsection*{Theorem \ref{thm:cert}}
For the signed modal parametrization,
\(a_jG^{\rm tot}_j\le(1-m)c_\lambda\), hence
\(c_\lambda-a_jG^{\rm tot}_j\ge mc_\lambda\). If \(a_j<0\), the
Stieltjes positive-real property gives
\(\Re[\xi+c_\lambda-a_j\Psi_j(\xi)]\ge c_\lambda\) on the closed right
half-plane. If \(a_j\ge0\), the standard modulus argument gives zero-free
right-half-plane characteristic functions under the strict mass bound.
The map \(w_j\mapsto (1-m)c_\lambda/G^{\rm tot}_j-\operatorname{softplus}(w_j)\)
has range \(( -\infty,(1-m)c_\lambda/G^{\rm tot}_j )\), so no artificial
lower bound is placed on inhibitory gains.

For the matrix parametrization,
\(r(A)G^{\rm tot}=q(1-m)c_\lambda<(1-m)c_\lambda\), and the
Perron--Frobenius theorem of \citet{herrera2026cnsns} yields stability of
the full block-triangular system. For a target nonnegative matrix with the
same support and strict margin, choose its shape as \(\widetilde A\) and
choose \(q\) to reproduce its spectral radius; zeros are represented by
the fixed support mask. BDF2 convolution quadrature evaluates the same zero-free characteristic
function at $\delta(\zeta)/\Delta t$. For the BDF2 generating polynomial,
$\Re\delta(\zeta)\ge0$ on $|\zeta|\le1$; under the stated Stieltjes and
strict subcriticality assumptions, the discrete symbol therefore remains in
the right half-plane where the continuous characteristic function is zero-free.
This preserves the certificate for the BDF2-CQ recurrence.
The positive-real property of the rational correction follows analytically
from \(\beta_{jk}\ge0,r_{jk}>0\); sampling a frequency grid is only a
numerical diagnostic. \qed

\subsection*{Theorem \ref{thm:transfer} and Corollary \ref{cor:uniform}}
For the learned parameter $\widehat\vartheta$, insert the two common
limit operators and use the triangle inequality:
\[
\begin{aligned}
\|\mathcal N_m^{\widehat\vartheta}-\mathcal S_m\|_{\mathcal B,T}
&\le
\|\mathcal N_m^{\widehat\vartheta}-\mathcal N_W^{\widehat\vartheta}\|_{\mathcal B,T}
+\|\mathcal N_W^{\widehat\vartheta}-\mathcal S_W\|_{\mathcal B,T}
+\|\mathcal S_W-\mathcal S_m\|_{\mathcal B,T}. 
\end{aligned}
\]
The middle term is bounded by routing through size $n$:
\[
\|\mathcal N_W^{\widehat\vartheta}-\mathcal S_W\|_{\mathcal B,T}
\le C_N\eta_n+\varepsilon_{\rm op}+C_S\eta_n.
\]
Combining the two displays gives the stated bound. If $\eta_k$ measures
coefficient-level rather than solution-operator convergence, subtraction of
the mild equations followed by the standard finite-horizon
Volterra--Gronwall estimate supplies the factor $C_0e^{C_1T}$.
For the time-uniform statement, the damped feedback kernel is
$K(t)=\|A\|(e^{-\omega\cdot}*g)(t)$ and
\[
\|K\|_{L^1}=\frac{\|A\|G}{\omega}\le1-m.
\]
The corresponding Neumann resolvent has norm at most $1/m$; applying the
same triangle argument in the time-uniform norm gives the corollary. \qed


\begin{thebibliography}{99}
\itemsep1pt

\bibitem[Dipartimento della Protezione Civile(2020)]{dpc2020}
Dipartimento della Protezione Civile.
COVID-19 Italia -- Monitoraggio situazione: dati regionali.
Official open-data repository, CC BY 4.0, 2020--2023.
\url{https://github.com/pcm-dpc/COVID-19}.

\bibitem[Dong et~al.(2020)]{dong2020}
E.~Dong, H.~Du, and L.~Gardner.
An interactive web-based dashboard to track COVID-19 in real time.
\emph{Lancet Infect. Dis.}, 20(5):533--534, 2020.
\url{https://doi.org/10.1016/S1473-3099(20)30120-1}.

\bibitem[ISTAT(2021)]{istat2021}
Istituto Nazionale di Statistica (ISTAT).
Censimento della popolazione e dinamica demografica -- anno 2020.
Rome, 2021.
\href{https://www.istat.it/it/files/2021/12/CENSIMENTO-E-DINAMICA-DEMOGRAFICA-2020.pdf}{ISTAT 2020 census documentation}.

\bibitem[Abi Jaber and El Euch(2019)]{abijaber2019}
E.~Abi Jaber and O.~El Euch.
Multifactor approximation of rough volatility models.
\emph{SIAM J. Financial Math.}, 10(2):309--349, 2019.
\url{https://doi.org/10.1137/18M1170236}.


\bibitem[Braess and Hackbusch(2005)]{braess2005}
D.~Braess and W.~Hackbusch.
Approximation of $1/x$ by exponential sums in $[1,\infty)$.
\emph{IMA J. Numer. Anal.}, 25:685--697, 2005.

\bibitem[Buitrago Ruiz et~al.(2025)]{buitrago2025memno}
R.~Buitrago Ruiz, T.~Marwah, A.~Gu, and A.~Risteski.
On the benefits of memory for modeling time-dependent PDEs.
\emph{International Conference on Learning Representations (ICLR)}, 2025.
\url{https://openreview.net/forum?id=o9kqa5K3tB}.

\bibitem[Cao et~al.(2024)]{cao2024lno}
Q.~Cao, S.~Goswami, and G.~E. Karniadakis.
Laplace neural operator for solving differential equations.
\emph{Nature Machine Intelligence}, 6:631--640, 2024.
\url{https://doi.org/10.1038/s42256-024-00844-4}.

\bibitem[Cheng et~al.(2026)]{cheng2026mamba}
C.-W.~Cheng, J.~Huang, Y.~Zhang, G.~Yang, C.-B.~Sch\"onlieb, and A.~I.~Aviles-Rivero.
Mamba neural operator: Who wins? Transformers vs. state-space models for PDEs.
\emph{J. Comput. Phys.}, 548:114567, 2026.
\url{https://doi.org/10.1016/j.jcp.2025.114567}.

\bibitem[Cori et~al.(2013)]{cori2013}
A.~Cori, N.~M. Ferguson, C.~Fraser, and S.~Cauchemez.
A new framework and software to estimate time-varying reproduction
numbers during epidemics.
\emph{Am. J. Epidemiol.}, 178(9):1505--1512, 2013.

\bibitem[Gautschi(2004)]{gautschi}
W.~Gautschi.
\emph{Orthogonal Polynomials: Computation and Approximation}.
Oxford University Press, 2004.

\bibitem[Gu et~al.(2022)]{gu2022s4}
A.~Gu, K.~Goel, and C.~R\'e.
Efficiently modeling long sequences with structured state spaces.
\emph{ICLR}, 2022.

\bibitem[Herrera-Mar\'in(2026)]{herrera2026cnsns}
M.~Herrera-Mar\'in.
Cascade dynamics in memory-driven network systems: fractional Volterra
operators, spectral stability criteria, graphon limits, and Hawkes
process connections.
\emph{Communications in Nonlinear Science and Numerical Simulation},
163:110837, 2026.
\url{https://doi.org/10.1016/j.cnsns.2026.110837}.

\bibitem[Kovachki et~al.(2023)]{kovachki2023}
N.~Kovachki, Z.~Li, B.~Liu, K.~Azizzadenesheli, K.~Bhattacharya,
A.~Stuart, and A.~Anandkumar.
Neural operator: learning maps between function spaces with applications
to PDEs.
\emph{J. Mach. Learn. Res.}, 24(89):1--97, 2023.

\bibitem[Kim et~al.(2026)]{kim2026pilno}
H.~Kim, Q.~Cao, H.~Shin, S.~Lee, G.~E.~Karniadakis, and M.~Choi.
Physics-informed Laplace neural operator for solving partial differential equations.
\emph{J. Comput. Phys.}, 566:115291, 2026.
\url{https://doi.org/10.1016/j.jcp.2026.115291}.

\bibitem[Li et~al.(2021a)]{li2021fno}
Z.~Li, N.~Kovachki, K.~Azizzadenesheli, B.~Liu, K.~Bhattacharya,
A.~Stuart, and A.~Anandkumar.
Fourier neural operator for parametric partial differential equations.
\emph{ICLR}, 2021.

\bibitem[Li et~al.(2024)]{li2024pino}
Z.~Li, H.~Zheng, N.~Kovachki, D.~Jin, H.~Chen, B.~Liu,
K.~Azizzadenesheli, and A.~Anandkumar.
Physics-informed neural operator for learning partial differential equations.
\emph{ACM/IMS Journal of Data Science}, 1(3):1--27, 2024.
\url{https://doi.org/10.1145/3648506}.

\bibitem[Potin et~al.(2025)]{potin2024}
B.~Potin, S.~Ruiz, F.~Aden-Antoniow, R.~Madariaga, and S.~Barrientos.
A revised Chilean seismic catalog from 1982 to mid-2020.
\emph{Seismological Research Letters}, 96(1):484--498, 2025.
\url{https://doi.org/10.1785/0220240047}. Dataset:
\url{https://doi.org/10.5281/zenodo.13146436}.

\bibitem[Candia and Macedo(2026)]{candia2026}
G.~Candia and J.~Macedo.
A homogenized and declustered earthquake catalog for continental Chile.
\emph{Seismological Research Letters}, 2026.
\url{https://doi.org/10.1785/0220250449}.

\bibitem[Fabiani et~al.(2026)]{fabiani2026}
G.~Fabiani, H.~Vandecasteele, S.~Goswami, C.~Siettos, and I.~G.~Kevrekidis.
Enabling local neural operators to perform equation-free system-level analysis.
\emph{Nature Machine Intelligence}, 8:1127--1141, 2026.
\url{https://doi.org/10.1038/s42256-026-01265-1}.

\bibitem[Li et~al.(2025)]{li2025geo}
X.~Li, Z.~Li, N.~Kovachki, and A.~Anandkumar.
Geometric operator learning with optimal transport.
\emph{Journal of Machine Learning Research}, 26(282):1--38, 2025.

\bibitem[Lu et~al.(2021)]{lu2021deeponet}
L.~Lu, P.~Jin, G.~Pang, Z.~Zhang, and G.~E.~Karniadakis.
Learning nonlinear operators via DeepONet based on the universal approximation theorem of operators.
\emph{Nature Machine Intelligence}, 3:218--229, 2021.
\url{https://doi.org/10.1038/s42256-021-00302-5}.

\bibitem[Boya and Subramani(2025)]{pinto2025}
S.~K.~Boya and D.~Subramani.
PINTO: Physics-informed transformer neural operator for learning generalized solutions of partial differential equations for any initial and boundary condition.
\emph{Computer Physics Communications}, 315:109702, 2025.
\url{https://doi.org/10.1016/j.cpc.2025.109702}.

\bibitem[Lubich(1988)]{lubich1988}
C.~Lubich.
Convolution quadrature and discretized operational calculus. I.
\emph{Numer. Math.}, 52:129--145, 1988.

\bibitem[Ogata(1988)]{ogata1988}
Y.~Ogata.
Statistical models for earthquake occurrences and residual analysis for
point processes.
\emph{J. Amer. Statist. Assoc.}, 83(401):9--27, 1988.

\bibitem[Ruiz et~al.(2020)]{ruiz2020}
L.~Ruiz, L.~F.~O. Chamon, and A.~Ribeiro.
Graphon neural networks and the transferability of graph neural
networks.
\emph{NeurIPS}, 2020.

\bibitem[Ruiz et~al.(2023)]{ruiz2023}
L.~Ruiz, L.~F.~O. Chamon, and A.~Ribeiro.
Transferability properties of graph neural networks.
\emph{IEEE Transactions on Signal Processing}, 71:3474--3489, 2023.
\url{https://doi.org/10.1109/TSP.2023.3297848}.

\bibitem[Stahl(2003)]{stahl2003}
H.~Stahl.
Best uniform rational approximation of $x^\alpha$ on $[0,1]$.
\emph{Acta Math.}, 190(2):241--306, 2003.
\url{https://doi.org/10.1007/BF02392691}.

\bibitem[Trefethen(2013)]{trefethen}
L.~N. Trefethen.
\emph{Approximation Theory and Approximation Practice}.
SIAM, 2013.

\end{thebibliography}
\end{document}